\documentclass{article}
\usepackage{dd_preprint,times}
\usepackage[T1]{fontenc}
\usepackage{microtype}
\usepackage{amsmath,amssymb,amsthm,bm}
\usepackage{booktabs,tabularx,multirow}
\usepackage{graphicx}
\usepackage{flafter}
\usepackage{placeins}
\usepackage{xcolor}
\usepackage{url}
\usepackage[hidelinks]{hyperref}
\definecolor{draftaccent}{HTML}{85501D}

\definecolor{scene}{HTML}{247C8B}
\definecolor{entity}{HTML}{C88028}
\definecolor{relation}{HTML}{A64D79}

\newcommand{\R}{\mathbb{R}}
\newcommand{\sg}{\operatorname{sg}}

\newcommand{\method}{Dancing Diffusion}
\newtheorem{proposition}{Proposition}
\newtheorem{lemma}{Lemma}
\newtheorem{corollary}{Corollary}
\title{Dancing Diffusion:\\Learning Continuous Programs\\for Image Generation and Editing}
\author{Cui Enhe\textsuperscript{1,2}\quad Jinyue Yang\textsuperscript{2}\quad Bojian Yin\textsuperscript{2}\quad Guoqi Li\textsuperscript{2}\\
\textsuperscript{1}Renmin University of China\\
\textsuperscript{2}Institute of Automation, Chinese Academy of Sciences}
\iclrfinalcopy
\hypersetup{pdftitle={Dancing Diffusion: Learning Continuous Programs for Image Generation and Editing},pdfauthor={Cui Enhe, Jinyue Yang, Bojian Yin, Guoqi Li},pdfsubject={Continuous conditioning programs; mechanism study and development evaluation}}
\begin{document}
\maketitle
\begin{abstract}
Diffusion models generate samples through a multi-step denoising trajectory,
yet condition every step on the same instruction, so the trajectory cannot
be addressed at the level of individual operations. We introduce \method{},
which assigns each span of timesteps a dedicated continuous condition block
and thereby turns conditioning into a latent program. A compiler maps text
features into this program, and a fixed pretrained generator executes its
blocks over successive sampling intervals. The trajectory can then dance,
keeping its executed history while its continuation is redirected. A shared
prefix is executed once and reused, and a later operation is rewritten while
content fixed by earlier operations is preserved. The compiler is trained
through the images it produces, with visual feedback on both the requested
change and the retained content. On SANA, reusing an executed prefix
preserves markedly more of the scene than regenerating from noise at a
similar strength of edit, learned conditions preserve more background than
native prompt switching from the same saved state, and the learned program
improves generation quality over native SANA. Later branches weaken the
edit, which leaves semantic control over individual operations limited. By
converting a single-instruction trajectory into a program of addressable
operations, \method{} lays a foundation for structured intermediate
computation in diffusion models.
\end{abstract}

\section{Introduction}
\label{sec:introduction}
Diffusion models generate samples through a multi-step denoising trajectory,
yet condition every step on the same instruction. Consequently, the
trajectory is not addressable at the level of individual operations. In
practice, variants that share early operations must recompute them, an edit
to one attribute propagates to content the request never mentioned, and
intermediate operations can be neither inspected nor corrected.

We introduce \method{}, which assigns each span of timesteps a dedicated
continuous condition block, turning conditioning into a latent program. A
compiler maps text features into this program, and a fixed pretrained
generator executes its blocks over successive sampling intervals
(Figure~\ref{fig:trajectory-topology}b). The trajectory thereby becomes
addressable at the level of individual operations, and it can dance, keeping
its executed history while its continuation is redirected. A shared prefix
can be executed once and reused, a later operation can be rewritten while
content fixed by earlier operations is preserved, and each block can be
intervened on in isolation to test what it controls
(Figure~\ref{fig:trajectory-topology}a).

\begin{figure}[t]
\centering
\includegraphics[width=\linewidth]{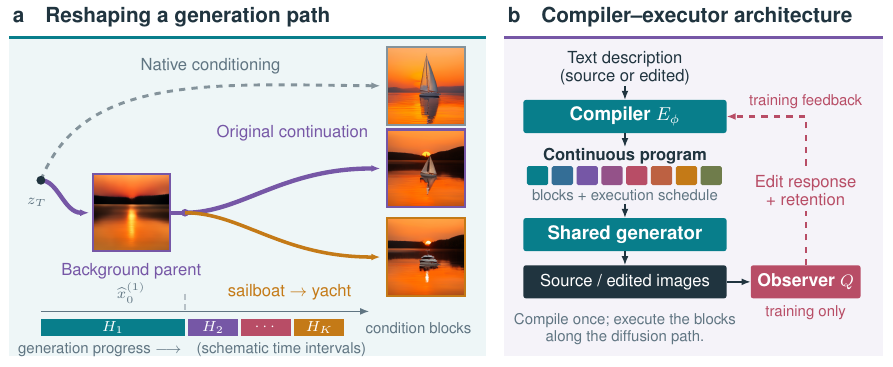}
\caption{\textbf{Dancing: editing the generation path.}
\textbf{(a)} A background-only first block produces a common parent, shown
as a predicted-clean readout. Both continuations reuse the same saved noisy
state and introduce a sailboat or yacht through future conditions. A native
same-noise generation is shown for orientation. Curves and intervals are
schematic; the images illustrate the operation and are not a final-checkpoint
comparison or evidence of native-method superiority.
\textbf{(b)} A compiler produces conditions for a shared generator, with
training-only feedback from a visual observer.}
\label{fig:trajectory-topology}
\end{figure}

We formalize when resuming a saved state reproduces the executed history
exactly and when a change to future conditions leaves declared background
content bounded. The compiler is trained through the images it produces.
A visual observer assesses both the requested change and the retention of
unrelated content, gradients update only the compiler, and the observer is
absent at inference. Temporal conditioning and learned connectors have
substantial precedents, including stage-specific prompt representations and
timestep-aware language interfaces
\citep{zhang2023prospect,hu2024ella,huberman2026sap}. Our contribution is the
combined compiler--executor formulation and its study under paired edit
feedback and controlled suffix interventions.

Our experiments test each of these properties directly. In a controlled
setting, learned conditioning changes when content appears, halving the
early subject error of a single-expert baseline while deferring background
detail. On SANA, resuming a saved state reproduces the source exactly, and
rewriting later operations leaves the executed history intact. At an early
branch, background LPIPS falls by up to 36\% relative to full regeneration,
from 0.321 to 0.227 at 6/12/12, while the semantic change is nearly
unchanged (0.088 versus 0.092). The reduction holds in every schedule under
paired family-bootstrap intervals and is confirmed on all 2376 outputs by
independent detector regions that the training observer never sees. From
the same saved state, learned conditions preserve more background than
native prompt switching in all three schedules, with comparable
replacement counts in two of them. These edits do not cost image quality. The learned program lowers
FID below native SANA, from 7.349 to 7.145 at 6/12/12, and a shared prefix is computed
once for all its continuations. Semantic control over individual operations,
however, remains limited, since later branches weaken the edit and
interventions on blocks confirm routing rather than a learned semantic
factorization. By converting a single-instruction trajectory into a program
of addressable operations, \method{} lays a foundation for structured
intermediate computation in diffusion models.

\section{Related work}
\label{sec:related-work}
\paragraph{Temporal conditioning and generation trajectories.}
Expert denoisers allocate capacity across noise regimes \citep{balaji2022ediffi},
while stage-specific embeddings, successive proxy prompts, and timestep-aware
connectors organize the information supplied to a generator
\citep{zhang2023prospect,huberman2026sap,hu2024ella}.
ProSpect optimizes stage-specific embeddings for reference-image attributes;
SAP provides training-free stage-aware prompting; ELLA learns a language
connector to a fixed diffusion model. These works show that the timing of
conditioning and a learned connector can each shape generation. Prompt-conditioned interventions
also study when concepts influence generation \citep{gorgun2025pci}.
We instead learn a caption-dependent continuous program from paired editing
feedback and use its saved intermediate states as points at which generation
can be redirected, comparing learned and native responses at identical
states.

\paragraph{Learned control interfaces and visual feedback.}
Prefix and prompt tuning adapt continuous inputs to language models;
concept embeddings and semantic connectors extend this approach to image
models \citep{li2021prefix,lester2021prompt,gal2022textual,hu2024ella}.
Spatial branches and gated grounding layers provide additional visual
control \citep{zhang2023controlnet,li2023gligen}. Policy-gradient and
differentiable-reward methods instead learn from generated outputs
\citep{black2023ddpo,fan2023dpok,clark2023draft,prabhudesai2024vader}.
We combine interface adaptation with feedback through a fixed generator,
using edit-response and retention objectives. The observer supplies training
gradients and is absent at inference, unlike sampling-time semantic guidance
\citep{chefer2023attend}. Fixed generator weights define the adaptation setting;
they do not prove that the complete behavior existed before interface training.

\paragraph{Image editing and semantic interventions.}
Perturbation and inversion methods start from an input image, attention
control constrains modified-caption generation, and instruction-conditioned
models learn paired transformations
\citep{meng2021sdedit,mokady2023nulltext,hertz2022prompt,brooks2022instructpix2pix}.
FlowEdit constructs an inversion-free mapping between source- and
target-conditioned distributions \citep{kulikov2025flowedit}.
Our interface resumes generated trajectories and replaces future conditions;
it does not accept arbitrary photographs through an inversion module.
We distinguish same-state suffix controls from image-input FlowEdit/SDEdit
adaptations, including their source access and compute. The SANA adaptations
do not reproduce published SD3/FLUX results. Interchange interventions motivate
testing whether low-level replacements realize high-level changes
\citep{geiger2022iit}; exact prefix reuse alone verifies execution, not semantics.

\paragraph{Evaluating semantic change and preservation.}
Generation benchmarks separate object presence, attribute binding, and
spatial relations \citep{ghosh2023geneval,huang2023compbench}. Editing
benchmarks additionally distinguish instruction adherence, content retention,
and perceptual quality through multimodal or regional assessments
\citep{ku2024viescore,yang2025complexedit,ye2025imgedit,jiang2026geditv2,qian2025giebench}.
Tool-assisted evaluations also document omissions and hallucinated changes
\citep{wang2025cigeval,yosef2025editinspector}.
We report semantics and preservation separately, retain missing-output and
order-disagreement rates for the adapted VLM, and treat CLIP and grounding as
proxies, particularly where they overlap the training observer
\citep{radford2021clip,liu2023groundingdino}.
CleanFID controls distribution-quality preprocessing \citep{parmar2022cleanfid};
it cannot establish editing correctness.

\section{Method}
\label{sec:method}
We edit a generation path by changing its future conditions while retaining
its executed history. This is the path's ``dancing'': temporal conditioning
becomes a program learned through execution and modified before resumption. The formulation
keeps three objects apart, namely a request, the conditions compiled from it,
and the sampler state reached under those conditions.

\subsection{Compiling a continuous program}
For a request $c$, let $H(c)\in\R^{L\times d}$ and $A(c)$ denote the text
features and attention mask accepted by the generator. A compiler maps them
to $K$ condition blocks and validity masks:
\begin{equation}
 (Z,M)=E_\phi(H(c),A(c)),\qquad
 Z\in\R^{K\times L\times d},\quad M\in\{0,1\}^{K\times L}.
 \label{eq:program}
\end{equation}
Each block retains the native token axis and enters the generator through
its conditioning interface. The program is $P(c)=(Z,M,\tau)$, where
$0=\tau_0<\cdots<\tau_K=1$ specifies successive execution intervals.
The representation permits arbitrary block counts and interval lengths;
the concrete architecture and schedule belong to the experimental instantiation.
Role-based supervision can initialize the blocks, but contextualized token
features need not contain exclusively the concepts assigned to their masks.
We therefore test block responses rather than assume semantic independence.

\subsection{Executing and observing a path}
A shared generator applies $N$ sampling updates from initial noise $x_0=\xi$:
\begin{equation}
 x_{n+1}=f_{\theta,n}(x_n;Z_{k(n)},M_{k(n)}),\quad
 k(n)=k\ \text{when }n/N\in[\tau_{k-1},\tau_k).
 \label{eq:execution}
\end{equation}
The final update belongs to the last interval. The compiler runs once per
request; execution determines which block is used at each step. We write
$F_{a:N}(x_a;P)$ for the remaining sampler updates and $D$ for image decoding.
A saved state $x_a$ contains sampling noise. To visualize a velocity-model
trajectory we instead decode its predicted-clean latent,
$\widehat x_{0,n}=x_n-\sigma_n v_\theta(x_n,\sigma_n;Z_{k(n)},M_{k(n)})$.
This readout describes the current prediction; it is not the state used to
resume the sampler.

\subsection{Intervening on future conditions}
\label{sec:branch}
Let a reference source program $\bar P(c)$ produce a source image $I_s$ and
saved state $\bar x_b$ after $b$ updates. For a target request $c'$, compile
$P_\phi(c')$ and replace conditions used from step $b$ onward:
\begin{equation}
 I_e^{(b)}=D\!\left(F_{b:N}(\sg[\bar x_b];P_\phi(c'))\right),
 \qquad I_s=D\!\left(F_{b:N}(\bar x_b;\bar P(c))\right).
 \label{eq:suffix-edit}
\end{equation}
Blocks that have already finished cannot affect the resumed computation.
An intervention can occur at a block boundary or a declared step within a
block; in the latter case it also replaces the remainder of the active block.
For $b=0$, the target program executes from the original noise, giving full
regeneration. For $b>0$, the source has already fixed a common history.

Under identical materialized inputs and deterministic execution, replaying
the source suffix reproduces the source endpoint exactly
(Appendix~\ref{app:proofs}). This protects the executed state, not future
pixels, since a new suffix can still move objects, change the background, or fail
to realize the target. Exact resumption and semantic preservation are therefore
separate tests. Native prompt switching can use the same saved state, making
it a direct control for the effect of learned suffix conditioning.
Appendix~\ref{app:factor-independence} formalizes background--subject
independence as selective endpoint response, distinct from statistical
independence of encoded features.

\subsection{Learning from an edit and its preservation cost}
\label{sec:learning}\label{sec:training}
Training uses the intervention in Eq.~\eqref{eq:suffix-edit}. The source
image, source state and preservation region are fixed references. A visual
observer $Q$ assesses the edited image and the detached source. For example,
with old- and new-entity evidence $q_{\rm old},q_{\rm new}$ and margin $m$,
response terms can take the form
\begin{align}
 \ell_{\rm gain}&=[m-(q_{\rm new}(I_e)-\sg[q_{\rm new}(I_s)])]_+,\\
 \ell_{\rm remove}&=[m-(\sg[q_{\rm old}(I_s)]-q_{\rm old}(I_e))]_+.
 \label{eq:response}
\end{align}
Terminal content terms assess the target endpoint, while preservation terms
penalize changes outside a declared source edit region. The objective is
\begin{equation}
 \mathcal L_{\rm edit}=\sum_{r\in\mathcal R}\lambda_r\ell_r/s_r,
 \qquad \lambda_r\geq0,\quad s_r>0.
 \label{eq:edit-objective}
\end{equation}
The coefficients and scales are fixed before the formal optimization run.
``Edit-only'' includes both requested change and content retention. It names
the adaptation objective, not an assumption about how the encoder was
initialized. Full-rollout paired training is the special case $b=0$; the
SANA experiments instead train on immutable source-prefix replays.

\begin{figure}[t]
\centering
\includegraphics[width=\linewidth]{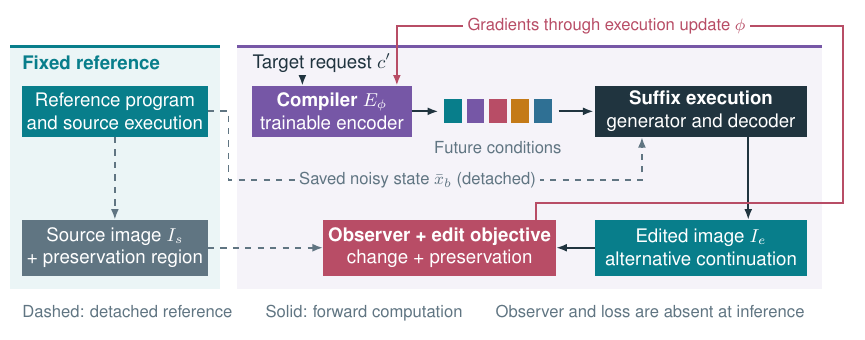}
\caption{\textbf{Training the conditions used after a saved prefix.}
The reference path supplies a detached sampler state and source image.
The target compiler provides future conditions; the suffix output receives
observer feedback for change and retention. Only compiler parameters are
optimized. The source reference cannot move to make the edit loss easier.
At deployment the observer and loss are absent.}
\label{fig:joint-training}
\end{figure}

\paragraph{Where learning occurs.}
With the schedule and Boolean masks held fixed, continuous values carry the
learning signal through execution:
\begin{equation}
 \nabla_\phi\mathcal L
 =\sum_{k\,\mathrm{active\ after}\ b}
 \left(\frac{\partial Z_k}{\partial\phi}\right)^{\!\top}
 \nabla_{Z_k}\mathcal L.
 \label{eq:compiler-learning-path}
\end{equation}
The generator, image decoder and observer can remain fixed while their
input derivatives train the compiler. The cached prefix is detached; discrete
masks and switching times receive no claimed gradient. Shared compiler
parameters mean that an update through a future block can also change the
compiler's earlier outputs on a fresh rollout. This motivates separate
checks of same-prompt drift and generation quality.

\subsection{Deployment and extensions}
At deployment, compile the request and run the generator, optionally caching
a prefix for several target continuations. No observer is queried. The same
interface can also receive trajectory supervision,
$\mathcal L_{\rm joint}=\mathcal L_{\rm edit}+\lambda_{\rm traj}\mathcal L_{\rm traj}/s_{\rm traj}$,
but the primary adaptation sets $\lambda_{\rm traj}=0$.
The optional extension and the current initialization history are documented
in Appendix~\ref{app:current-implementation}; they are not assumed to improve
editing. The mechanism study separates consequences of the execution
interface from consequences of learning its conditions.

\section{Experiments}
\label{sec:experiments}
We first test in a controlled setting whether learned conditioning can
change when content appears, then ask on SANA what changes when future
conditions are replaced. The two settings differ in data, supervision and
readouts, so their numbers are not pooled.

\subsection{A controlled setting: when content appears}
\label{sec:toy}
\begin{figure}[t]
\centering
\includegraphics[width=\linewidth]{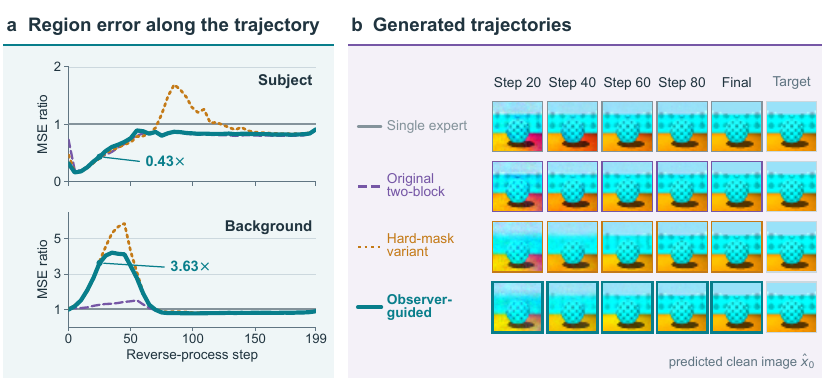}
\caption{\textbf{Learning when to reconstruct foreground and background.}
\textbf{(a)} Exact-mask subject and background MSE on 1,024 validation
images, normalized pointwise by the single-expert MSE, so that one denotes
parity and lower is better. Every measured point is retained, and absolute
curves are in Appendix~\ref{app:toy-absolute-mse}. \textbf{(b)} Predicted-clean
$\hat{x}_0$ readouts of separate pure-noise DDPM-200 trajectories with common
initial and per-step noise, followed by the target condition image. Row
colors match the curves in (a). The observer-guided model delays wall texture
while keeping a recognizable sphere. These trajectories are not samples from
the noised-reference evaluation in (a), and the example was selected post hoc
for correct colors, so it does not estimate how often colors are correct.}
\label{fig:toy-mse-trajectory}
\end{figure}

\paragraph{Question and controlled setting.}
Can learned conditioning change \emph{when} scene content becomes visible?
We study Patterned-Sphere3D, a controlled Shapes3D \citep{burgess2018shapes3d} derivative with exact
foreground masks and independently varied object and wall patterns.
The dataset contains 120,000 images at $64\times64$ resolution. We compare
an archived single-expert baseline, the original shared-denoiser two-block
model, a hard factor-mask variant, and the observer-guided two-block model
with texture supervision. All curves use the same single-expert reference;
these are historical model variants, not a one-variable ablation series.

\paragraph{Learning stage behavior through feedback.}
The observer-guided model shares one denoiser across stages and keeps the
full factor input available to its learned conditioning slots, with hard
factor visibility masks disabled. A frozen factor observer supplies semantic
feedback, with its factor-loss gradients restricted to the conditioning
encoder. Thus the encoder can learn stage-dependent semantic behavior
without manually removing factors from its input. The stage schedule and
supervised responsibilities are specified by the training objective;
regional reconstruction and texture losses additionally encourage early
object reconstruction and later wall detail. This is learned semantic
allocation under explicit supervision, rather than an independently
verified claim that observer feedback is its sole cause.

\paragraph{Region-wise evaluation.}
For 1,024 validation images and 41 noise levels, both methods receive the
same noised reference $(x_t,t)$ and noise realization. We compare their
clipped predicted-clean images with the clean reference, using the exact
sphere mask for subject MSE and its complement for background MSE.
Errors are averaged over RGB channels and region pixels in the $[-1,1]$
image range. The horizontal axis is reverse-process position $s=199-t$;
these curves are matched-state reconstruction diagnostics, not errors
accumulated along a generated trajectory. Separate pure-noise DDPM-200 \citep{ho2020ddpm}
readouts (Figure~\ref{fig:toy-mse-trajectory}b) show how the corresponding
models generate a scene over time.

\paragraph{Result and implication.}
The two regions follow different error profiles
(Figure~\ref{fig:toy-mse-trajectory}). At step 25, the observer-guided
model's subject MSE is $0.435\times$ the single-expert baseline, whereas
its background MSE is $3.629\times$ the baseline. By step 85, these ratios
are $0.870$ and $0.791$, respectively; at step 199 they are $0.912$ and
$0.897$. The early background penalty is consistent with delayed detail,
not an early quality improvement. The trajectory rows make the intended
effect visible, as the sphere is established before the wall pattern sharpens.
The displayed example has consistent final colors, but was selected for
that property; background-color failures elsewhere
remain, and low reconstruction error at late noise levels does not
rule them out. The controlled study therefore
supports changing the temporal allocation of visual content through
learned conditioning and explicit feedback. It motivates the SANA
prefix-reuse experiments without establishing universal generation quality
or attributing the effect to the observer alone.

\subsection{Editing text-to-image generation with SANA}
\label{sec:large-study}
\label{sec:v10-results}
\paragraph{Setup and primary comparison.}
We use SANA-600M \citep{xie2024sana} with a 273.1M-parameter compiler and three condition blocks.
The primary method is edit-only training, evaluated for the 3/9/18,
6/12/12, and 9/12/9 execution schedules. The final checkpoints are fixed
when training completed, before the present evaluation.
The 3/9/18 arm starts edit-only adaptation from initialization; the other
two arms inherit a 40-update joint warm start. This difference limits causal
comparisons across schedules. Their initialization histories are retained in
Appendix~\ref{app:current-implementation}.
The generator is fixed; the observer does not participate in generation.

The principal comparison is \emph{full regeneration}: apply the target
conditions from the initial noise, using the same trained compiler,
sampler, schedule, and seed as the editing arm. This is branch point zero.
Our intervention instead resumes the stored source state and changes only
future conditions. Comparing these arms asks whether reusing an executed
prefix preserves scene structure while achieving the requested change.
It does not, by itself, isolate the benefit of compiler training.
Initialized programs test the effect of adaptation; native full-caption
suffixes test the conditions used at a matched state
(Appendix~\ref{app:native-suffix}). These controls answer different questions
from full regeneration. The source paths in the matched comparisons come
from the frozen initial compiler, not the target-trained compiler.

\paragraph{Evaluation and interpretation.}
The training-disjoint evaluation uses 24 request families and three common
seeds, with branch points 0, 9, 15, and 21; all branch points are retained
rather than selected by their results. Identity resumption must reproduce
the uninterrupted source. Background LPIPS \citep{zhang2018lpips} and pixel measurements compare
each edit with its actual generated source, outside a fixed source-derived
edit region. Every conditional metric retains its accepted-mask denominator.
Semantic CLIP measurements are observer-family proxies, while the local
VLM protocol supplies explicitly adapted automated judgments, not human
preferences. Better preservation is useful only together with adequate
instruction following.
Complete final measurements and their coverage are collected in
Appendix~\ref{app:v13-final-evaluation}. Additional image-input controls and
independent detector measurements are separated by source access and
validation status in Appendix~\ref{app:comparative-evaluation}.

At branch 9, prefix reuse lowers background LPIPS relative to full
regeneration for all three schedules, while retaining a similar mean
semantic margin change. For example, at 6/12/12, LPIPS decreases from
0.321 to 0.227 and the CLIP margin change moves from 0.092 to 0.088.
Later branching strengthens preservation but eventually weakens the edit:
at branch 21, LPIPS is 0.033 and the margin change is 0.029.
These means support a controllable preservation--change
tradeoff, not superiority of the trained encoder over all conditioning
controls.
Paired family-bootstrap analysis confirms lower background LPIPS at branch 9
for all three schedules (Appendix~\ref{app:paired-prefix}); the reductions
in valid-attempt means are 36.0\%, 29.1\%, and 23.6\%, respectively.
The semantic proxy decreases for the first two schedules. Its interval
includes zero for 9/12/9, which does not establish equal editing success.
Independent category checks on 12 eligible families likewise show no uniform
ordering across schedules and detectors (Appendix~\ref{app:comparative-evaluation}).

\begin{table}[tbp]
\centering\small
\setlength{\tabcolsep}{4pt}
\caption{Matched category-cohort comparison: 12 families, three seeds, 36 attempts per schedule. Cells give BG LPIPS$\downarrow$ / DINO replacement count$\uparrow$ / OWLv2 replacement count$\uparrow$. Mask coverage is 32/36, 33/36 and 33/36; counts use all attempts. Regeneration starts at 0; others at 9. Native uses the saved state; SANA-adapted image editors use the source PNG. Counts are proxies, not full-caption success; all starts appear in Appendix~\ref{app:comparative-evaluation}.}
\label{tab:compact-independent-comparison}
\begin{tabular}{lccc}
\toprule
Method & 3/9/18 & 6/12/12 & 9/12/9\\
\midrule
Full regeneration & 0.383 / 13 / 27 & 0.334 / 11 / 29 & 0.307 / 14 / 29\\
Learned suffix & 0.254 / 14 / 29 & 0.247 / 15 / 27 & 0.242 / 14 / 25\\
Native suffix & 0.269 / 14 / 28 & 0.274 / 14 / 26 & 0.269 / 17 / 27\\
SDEdit--SANA & 0.284 / 16 / 28 & 0.288 / 16 / 27 & 0.293 / 12 / 27\\
FlowEdit--SANA & 0.143 / 13 / 27 & 0.138 / 12 / 25 & 0.142 / 12 / 24\\
\bottomrule
\end{tabular}
\end{table}

Table~\ref{tab:compact-independent-comparison} places the image-input controls alongside
regeneration and learned suffixes on the same category cohort. At start 9,
learned suffixes retain more background than SDEdit--SANA in all three
schedules; FlowEdit--SANA retains still more, with different replacement
counts. These are operating-point tradeoffs across different source access,
not an overall ranking. All starts and paired uncertainty remain in
Appendix~\ref{app:comparative-evaluation}.
An independent source-box check over all 2376 outputs confirms the same
preservation ordering at three region paddings. Learned branch 9 changes
less background than regeneration and SDEdit, but more than FlowEdit
(Appendix~\ref{app:independent-box-regions}). These source-only Grounding DINO
regions do not use the training observer and are box complements, not segmentations.

\paragraph{What changes beyond ordinary prompt switching?}
At exactly the same saved source state, native full-caption conditions already
support substantial subject replacement. Both use the same path-editing
operation; this comparison isolates the effect of learned conditions on
the resulting preservation--change balance.
The completed native extension covers all three schedules (432 additional
outputs with exact source and prefix matches). At branch 9, source-box
background LPIPS is lower with learned conditions by 0.017, 0.022 and 0.030
for 3/9/18, 6/12/12 and 9/12/9 (5\% padding; 65/72, 62/72 and 62/72
backgrounds). The descriptive family-bootstrap intervals are below zero;
all paddings and branches are retained in Appendix~\ref{app:native-source-box-regions}.
Late branching reverses this preservation advantage. At branch 9 on the fixed 36-attempt category cohort,
learned/native replacement counts are 14/14, 15/14 and 14/17 for Grounding
DINO, and 29/28, 27/26 and 25/27 for OWLv2, in schedule order. Thus the
first two schedules have similar category counts, while 9/12/9 favors native
conditions. The evidence supports an early-branch preservation benefit at
comparable replacement counts. The branch-point curves
(Figures~\ref{fig:development-tradeoff} and~\ref{fig:independent-tradeoff})
make this partial separation of background retention and subject change
visible. It is a behavioral property, not proof that encoder pretraining
caused a disentangled representation.

\paragraph{Qualitative comparison.}
\begin{figure}[!tbp]
\centering\setlength{\tabcolsep}{1pt}
\begin{tabular}{ccccc}
\scriptsize Source & \scriptsize Full regeneration & \scriptsize Learned suffix & \scriptsize FlowEdit--SANA & \scriptsize SDEdit--SANA\\
\multicolumn{5}{l}{\scriptsize Red bus to red truck; seed 101}\\
\includegraphics[width=.19\linewidth]{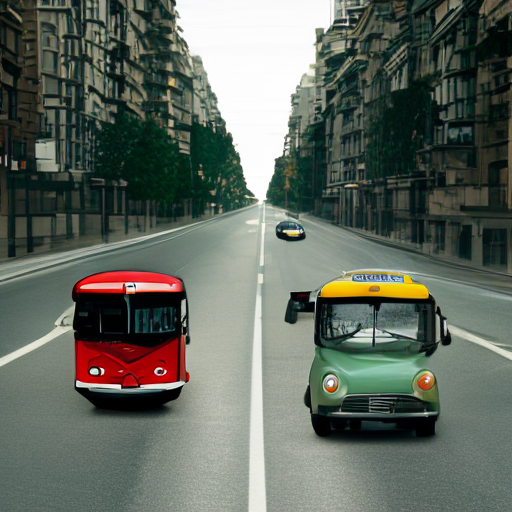} & \includegraphics[width=.19\linewidth]{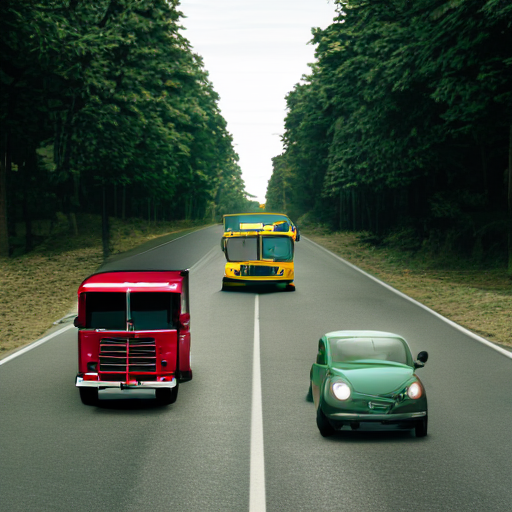} & \includegraphics[width=.19\linewidth]{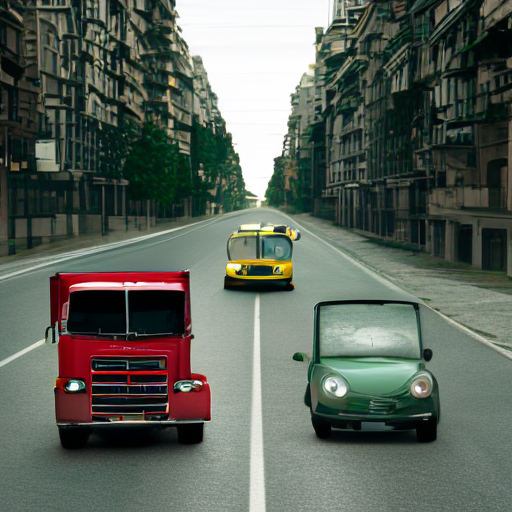} & \includegraphics[width=.19\linewidth]{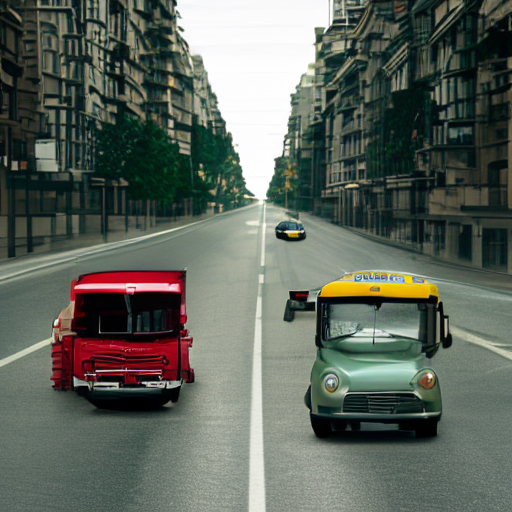} & \includegraphics[width=.19\linewidth]{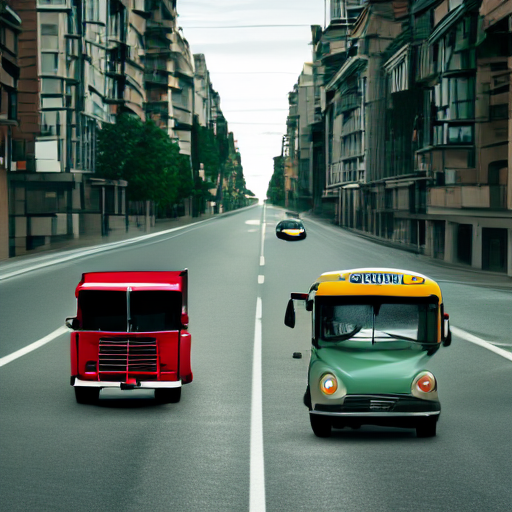}\\[3pt]
\multicolumn{5}{l}{\scriptsize Bench to planter; seed 202}\\
\includegraphics[width=.19\linewidth]{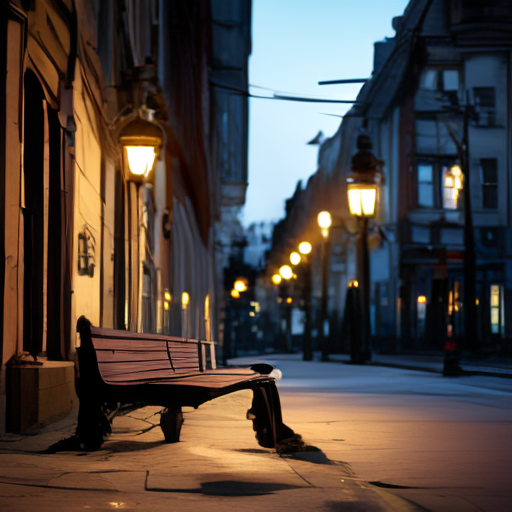} & \includegraphics[width=.19\linewidth]{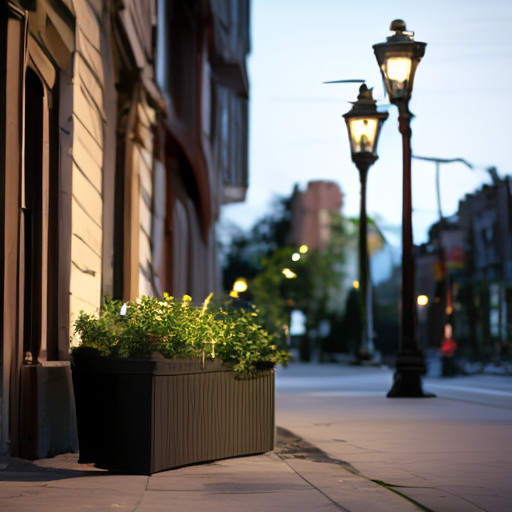} & \includegraphics[width=.19\linewidth]{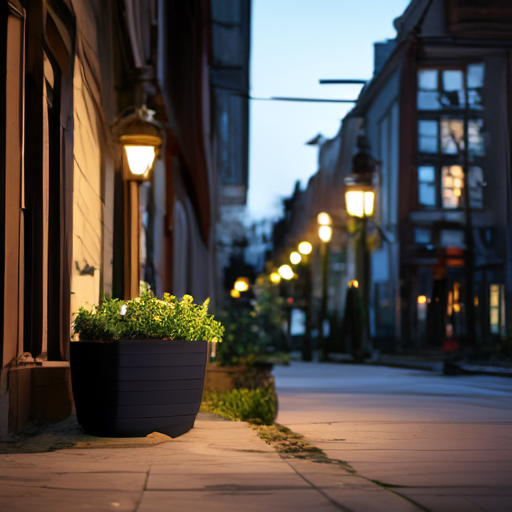} & \includegraphics[width=.19\linewidth]{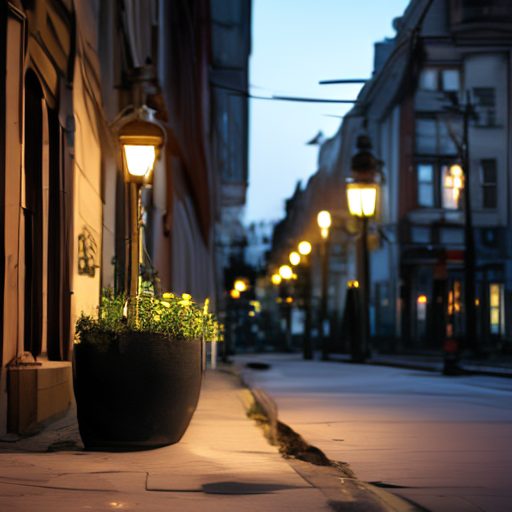} & \includegraphics[width=.19\linewidth]{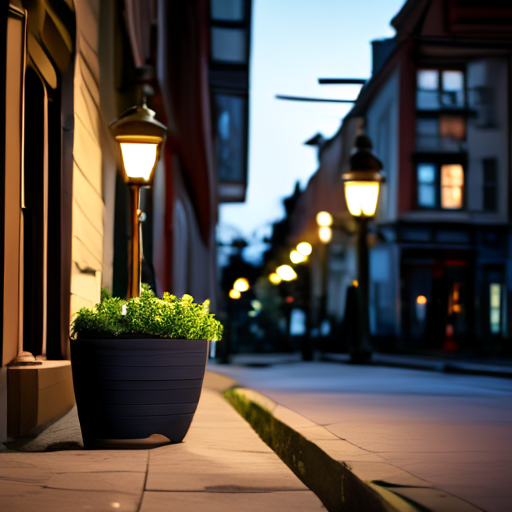}\\[3pt]
\end{tabular}
\caption{\textbf{Editing with visible scene retention.} Genuine, uncropped outputs for the same source and request: final checkpoint, schedule 6/12/12, branch/start 9 (regeneration: 0). The learned suffix retains avenue and building structure; image-input SANA adaptations also retain substantial context. Their source access differs (saved state versus PNG). These selected illustrations do not estimate success rates; failure cases are described in Appendix~\ref{app:final-representative}.}
\label{fig:v13-advantage-comparison}
\end{figure}

Figure~\ref{fig:v13-advantage-comparison} adds matched image-input controls
to the regeneration/prefix comparison. The avenue and street-lamp geometry
make preservation visible alongside the requested replacement.
Appendix~\ref{app:final-representative} gives more examples and describes failures,
including a kayak edit that misses the requested holding relation.
Aggregate comparisons are reported above.

\paragraph{What the mechanism tests establish.}
Block interventions distinguish routing from specialization. Across all
measured branches, 1200 interventions on already inactive blocks leave the
latent endpoint exactly unchanged relative to the same-compiler continuation.
Active-block semantic responses change with execution timing, but similar
patterns are present at initialization. Schedule crossing also changes the
relation between execution time and the compiler's learned clock. These tests
verify controllable routing without isolating a training-induced semantic
factorization. The complete schedule, source-path and block comparisons are
in Appendix~\ref{app:mechanism-current}.

Caching a common prefix also reduces the cost of four continuations by
22.0--23.4\% with identical endpoints, a benefit available to native prompt
switching as well (Appendix~\ref{app:prefix-cost}).

\paragraph{Generation quality.}
\label{sec:comparative-design}
All entries use MJHQ-30K \citep{li2024playground}, CleanFID 0.1.35 clean, 512 pixels, 30 steps, CFG 4, the same metadata ordering and CPU noise rule. Initializations are role-fitted encoders without edit-only adaptation. Final checkpoints are taken at training completion (update 1423).\par
\begin{table}[!htbp]\centering\small
\caption{Generation FID versus native SANA and across training stages. Negative final-minus-native differences indicate improvement. All entries use the same noise set; no training-seed uncertainty is inferred.}\label{tab:v13-final-quality}
\begin{tabular}{lrrrr}\toprule
Method / schedule & Initialization & Update 600 & Final 1423 & $\Delta$ vs. native\\\midrule
Native SANA & \multicolumn{3}{c}{7.349 (no adaptation)} & ---\\\midrule
3/9/18 & 7.248 & 7.203 & 7.315 & -0.034\\
6/12/12 & 7.268 & 7.247 & 7.145 & -0.204\\
9/12/9 & 7.411 & 7.466 & 7.375 & +0.026\\
\bottomrule\end{tabular}\end{table}
\par

Editing does not come at the expense of generation quality. The learned
program lowers FID below native SANA (7.349) by 0.204 at 6/12/12 (2.78\%)
and by 0.034 at 3/9/18, and stays within 0.026 at 9/12/9.
Relative to each initialization, changes are $+0.067$, $-0.123$ and
$-0.036$, respectively. Thus quality relative to the native generator and
the effect of edit training are distinct. These single-run results use
terminal checkpoints, not FID-selected checkpoints. FID measures generation
quality, not editing success.

\FloatBarrier
\section{Discussion and conclusion}
\label{sec:conclusion}
\method{} makes the generation path itself editable, with learned stage
conditions as a trainable interface for redirecting its future.
The controlled study demonstrates stage behavior
under explicit supervision. In SANA, early prefix reuse preserves more
background than full regeneration; later branching restricts the requested
change. Native suffix comparisons expose a preservation--replacement
tradeoff, and the learned program lowers FID below native SANA in two of
three schedules.
Caching saves computation with identical endpoints, as it would for
ordinary prompt switching.

\paragraph{Capability access through a trainable interface.}
The behavior belongs to the compiler--generator system, and fixed generator
weights cannot distinguish learned control from improved access to existing
behavior. Appendix~\ref{app:elicitation} specifies tests for these alternatives
and discusses the hypothesis connecting interface training to LLM post-training
\citep{yue2025rlvr,liu2025prorl}. Organizing sampling as a learned sequence
of decisions motivates the method; reasoning ability is not evaluated here.

\paragraph{Scope.}
The compiler runs once per caption, and the SANA results come from a
training-disjoint evaluation set with one training run per schedule. CLIP and source masks
share their family with the training observer, so we complement them with
independent detectors and observer-independent regions. Extending the
program to be recompiled during generation, and to image inputs through
inversion, are natural next steps.

\label{sec:main-end}
\clearpage
\subsection*{Reproducibility statement}
The appendix specifies sampling settings, model revisions, training weights,
checkpoint hashes, route definitions, and the family-level bootstrap unit.
Measured numerical figures are generated directly from archived machine-readable
reports. The preliminary long-run study separates selection diagnostics from
independent reporting; unmeasured results are not assumed. Schematic figures are labeled separately from measured results.
Model/data access and the complete experiment environment remain necessary
to reproduce generation; the manuscript does not report a fresh rerun of the
archived experiments.

\subsection*{AI use statement}
Generative AI tools assisted with implementation, research discussion, formal arguments,
interpretation of experiment reports, manuscript drafting and revision,
literature lookup, and creation of vector diagrams and plotting code.
A Qwen model supplied role-span pseudo-labels for the compiler's program data,
and a local Qwen vision--language model served as the automated pairwise judge
in the appendix evaluation.
All displayed numerical values come from experiment reports. The schematic diagrams are not generated-image evidence. Responsibility for the
paper's claims, citations, and artifacts rests with the authors.

\subsection*{Ethics statement}
The method inherits the content and bias risks of its pretrained generator
and visual evaluators. It provides no guarantee of factuality or fair
performance across concepts. Our experiments concern generated trajectories;
we do not evaluate manipulation of photographs of identifiable people.

\bibliographystyle{iclr2027_conference}
\bibliography{references}
\clearpage
\appendix
\section{Appendix roadmap}
\label{app:roadmap}
Part I records the implementation, the execution and trainability guarantees,
and the comparison protocols. Part II gives absolute errors for the controlled
setting. Part III reports the SANA editing evaluation, including paired
uncertainty, the native-suffix control, observer-independent regions and
prefix-reuse cost. Part IV examines schedules, blocks and an explicit-background
input variant. Part V contains a bidirectional stress test and additional
examples.

In the SANA measurements, detector counts and CLIP
margins are proxies, source boxes are not segmentations, and intervals carry
no multiplicity correction. Two outcomes
are negative. The automatic scene compiler fails the input-equality gate
(Appendix~\ref{app:b0-isolation}), so no FID is claimed for that variant, and
the large-subject stratum is empty on the recomputed sources, so that
comparison is unavailable.

\part*{Part I\quad Setup: implementation, guarantees and protocols}
\section{Current implementation and experimental provenance}
\label{app:current-implementation}
\paragraph{Compiler and executor.}
The SANA study uses three continuous blocks at native feature width 2304.
The 273.1M-parameter compiler is initialized from role-mask fitting at
update 800. A frozen copy of that initial compiler predicts the validity
masks and generates the source images and saved states, so the training
target does not move with later compiler updates. The SANA-600M 512-pixel
generator uses 30 steps and guidance 4. No trained decoder adapter is loaded
for the comparisons.

\paragraph{Training and warm start.}
The replay cohort has 9226 source-image identities, shared across the three
schedules. The 3/9/18 run starts at update zero from the role-fitted compiler,
without a joint parent or optimizer state. The 6/12/12 and 9/12/9 runs resume
their joint checkpoints at update 40, including Adam state, and continue with
zero trajectory weight to update 1423, giving 1383 edit-only updates.
We refer to the terminal update-1423 checkpoint as the final checkpoint.
Cross-schedule differences therefore combine schedule and initialization
history. The main comparisons use the terminal checkpoint; update 600 is an
intermediate diagnostic. The encoder learning rate is $2\times10^{-6}$, with
global batch 16, microbatch 4, weight decay 0 and training seed 20260930.
Each configuration is a single training run.

Training branches after the first block, at steps 3, 6 and 9 for schedules
3/9/18, 6/12/12 and 9/12/9. The two remaining blocks receive the target
conditions, and all 27, 24 and 21 suffix updates are differentiated with
activation checkpointing. The source image, prefix, source-region mask and
frozen validity masks are detached. Evaluation at steps 9, 15 and 21 is a
separate common-step comparison, not three training branch points.

\paragraph{Objective.}
Terminal and response terms use the frozen observer, and each raw term is
divided by its recorded frozen scale. The weights are target presence 2, old
absence 1, preservation 4, caption 0.4, background retention 1, response
presence 0.1, response absence 0.05, response heat 0.5 and response heat
absence 0.5. Relation and trajectory weights are zero. The source region uses
a 0.15 detection threshold and 0.02 padding, and accepted edit-area fractions
lie between 0.002 and 0.8. Exact configurations, recorded loss scales and
checkpoint identities are indexed in the reproducibility bundle.

\paragraph{Data and reporting scope.}
Training and evaluation indices are distinct, and
identity checks compare source-image identifiers rather than file positions.
The 24-family evaluation set uses seeds 101, 202 and 303. Every attempt is
retained for semantic proxies, and missing source masks stay in the coverage
count. Automatically substituted target captions can retain inappropriate
actions or attributes, so category presence alone cannot certify full-caption
compliance.

\paragraph{Optional joint extension.}
The joint objective adds a normalized velocity loss on a declared source
path with a detached target. The final edit-only results do not use it.
Historical joint runs and earlier pilots are archived separately and are not
pooled with the current comparisons.

\section{Execution guarantees and what they do not imply}
\label{app:proofs}
\begin{proposition}[Protected-prefix equality]
\label{prop:prefix}
With identical initial state, materialized prefix conditions, model parameters,
schedule and deterministic numerical execution, two runs agree at every step
of their common prefix. Saving and restoring the complete sampler state
reproduces the endpoint under the unchanged suffix.
\end{proposition}
\begin{proof}
Each update in Eq.~\eqref{eq:execution} has identical inputs in both runs, so
induction gives equality of the next state. The same induction starts at a
restored state for identity continuation. A stochastic sampler additionally
requires the same random state, and a multistep solver its history. Equality
across numerical runtimes is not asserted.
\end{proof}
The guarantee applies equally to native prompt switching. It is a
computational invariant and says nothing about preservation under a changed
suffix. Our endpoint-equality and inactive-block tests check it in the
deterministic implementation.

\paragraph{Differentiating a resumed program.}
The reference state in Eq.~\eqref{eq:suffix-edit} is detached, so the loss
differentiates the executed suffix and its continuous conditions but cannot
change the cached source path. A block whose interval ends before $b$ has no
direct effect on the loss, although shared encoder parameters can still
change its output after optimization. Frozen Boolean masks remove mask
selection from the learning path but not derivatives through valid
continuous features.

\subsection{Background--subject independence under path edits}
\label{app:factor-independence}
Fix a source state $\bar x_b$, schedule, masks and numerical randomness,
collectively $c$. Let $z$ contain the continuous suffix conditions, and define
the endpoint and its declared background and subject readouts by
\begin{equation}
 G_b(z;c)=D(F_{b:N}(\bar x_b;z)),\qquad
 h_B(z;c)=\rho_B(G_b(z;c)),\quad h_S(z;c)=\rho_S(G_b(z;c)).
 \label{eq:factor-readouts}
\end{equation}
The readouts specify the factors under test and need not exhaust image
semantics. Occlusion, shadows or relations can couple the two factors, so
independence is relative to the task, not a universal property of scenes.

\paragraph{Statistical and interventional independence.}
For a hypothesized factorization $z=(z_B,z_S)$, statistical independence
$p(z_B,z_S\mid c)=p(z_B\mid c)p(z_S\mid c)$ says nothing by itself about
which image features depend on each component. Independent scalars can still
give the entangled readout $h_B(z)=z_B+z_S$, and correlated coordinates can
drive separate output factors. Our temporal blocks are not assumed to be
$z_B$ and $z_S$.

What matters for a subject edit is a small background response together with
a realized target. For source conditions $z$ and edited conditions $z'$, let
\begin{equation}
 e_B=\|h_B(z';c)-h_B(z;c)\|,\qquad
 e_S=d_S(h_S(z';c),s^\star),
 \label{eq:factor-edit-errors}
\end{equation}
where $s^\star$ is the declared subject target. Selective editing requires
$e_B\le\epsilon_B$ and $e_S\le\epsilon_S$ for thresholds fixed in advance.
An unchanged image makes $e_B$ zero while failing the target, so background
retention alone is insufficient.

\begin{proposition}[Bounded cross-factor response]
\label{prop:factor-response}
Suppose $h_B$ is continuously differentiable along the admissible segment
$(z_B,z_S+t\delta_S)$, $t\in[0,1]$, and its cross-factor Jacobian has
operator norm at most $\kappa_{BS}$ on that segment. Then
\begin{equation}
 \|h_B(z_B,z_S+\delta_S;c)-h_B(z_B,z_S;c)\|
 \le \kappa_{BS}\|\delta_S\|.
 \label{eq:factor-cross-bound}
\end{equation}
The analogous bound with $\kappa_{SB}$ holds for the subject response to a
background edit, and vanishing cross-factor Jacobians on a convex product
domain make each readout invariant to the other coordinate.
\end{proposition}
\begin{proof}
The background difference equals
$\int_0^1 (\partial h_B/\partial z_S)(z_B,z_S+t\delta_S;c)\delta_S\,dt$,
and the operator-norm bound gives Eq.~\eqref{eq:factor-cross-bound}.
Interchanging the factors gives the second bound.
\end{proof}
The statement bounds leakage, not attainment of $s^\star$, and holds only at
a fixed source state. Background LPIPS and subject-replacement proxies
approximate the two sides of Eq.~\eqref{eq:factor-edit-errors} at the
endpoint, rather than measuring a cross-factor Jacobian directly.

\paragraph{Interpreting capability access.}\label{app:elicitation}
A fixed generator and a trained interface form a new composite system.
Separating improved access to a pretrained behavior from newly learned
control relies on a verified fixed executor together with native-prompt and
initialized-interface controls, both of which our experiments include. Finite
native-prompt search cannot show that an unobserved behavior is absent, and
the analogy to post-training in language models is a motivation rather than
a tested mechanism.

\FloatBarrier
\section{Trainability of the compiled interface}
\label{app:trainability}
The results below concern optimizing the interface rather than executing it
(Appendix~\ref{app:proofs}). They show that the learning problem is well posed
and that its reported objective cannot be gamed by its own normalization.
They do not by themselves predict that optimization succeeds.

\paragraph{Setting.}
Fix the generator $\theta$, decoder $D$, observer $Q$, schedule
$0=\tau_0<\cdots<\tau_K=1$ and Boolean masks $M$, and optimize only the
compiler parameters $\phi$. For a replay with source state $\bar x_b$, source
image $I_s$ and keep-mask $m$, the edited endpoint is
$I_e=D(F_{b:N}(\sg[\bar x_b];Z_\phi,M))$ as in Eq.~\eqref{eq:suffix-edit},
and the objective is Eq.~\eqref{eq:edit-objective} with weights
$\lambda_r>0$ and scales $s_r>0$ fixed before the run.

\subsection{The objective is well posed}
\begin{proposition}[Nonnegativity and a uniform lower bound]
\label{prop:well-posed}
Every deployed term is nonnegative, hence $\mathcal L_{\rm edit}\ge 0$
uniformly in $\phi$. Each hinge term is exactly $0$ when its margin is met,
and each presence term is bounded above by $-\log\varepsilon$.
\end{proposition}
\begin{proof}
The deployed terms are $-\log p$ and $-\log(1-p)$ with $p$ clamped to
$[\varepsilon,1-\varepsilon]$, $\varepsilon=10^{-6}$; a masked mean squared
error; $\operatorname{softplus}(\cdot)$; $\operatorname{relu}(m-g)$; and
$(1-s_t)+\operatorname{relu}(0.05+s_s-s_t)$ with $s_\bullet\in[-1,1]$ inner
products of unit-normalized CLIP features. Each is nonnegative, and a positive
combination of nonnegative terms is nonnegative, independently of $\phi$.
The hinge vanishes iff $g\ge m$, and the clamp bounds presence by
$-\log\varepsilon\approx13.82$ (the caption term by $4.05$).
\end{proof}
The proposition concerns the loss surface, not reachability. The terms
compete, so a common minimizer need not exist, and the softplus relation term
is never exactly zero (though of order $10^{-12}$ at the best attainable
margin). What it rules out is an apparent decrease caused by an unbounded
term.

\paragraph{Where the gradient can live.}
Eq.~\eqref{eq:compiler-learning-path} sums only over blocks active after $b$.
The next lemma makes this exclusion exact.
\begin{lemma}[Exact gradient support]
\label{lem:gradient-support}
Let block $k$ have execution interval $[\tau_{k-1},\tau_k)$. If
$\tau_k N\le b$, then $\mathcal L_{\rm edit}$ is constant in $Z_k$, so
$\nabla_{Z_k}\mathcal L_{\rm edit}=0$ identically.
\end{lemma}
\begin{proof}
By Eq.~\eqref{eq:execution} step $n$ reads only $Z_{k(n)}$. If
$\tau_kN\le b$ then $k(n)\ne k$ for every $n\ge b$, and the prefix is
detached, so nothing in the loss depends on $Z_k$.
\end{proof}
For the $6/12/12$ schedule at $N=30$ the block boundaries are steps 6, 18 and
30, so at $b=6$ the first block receives exactly zero gradient. This is the
formal counterpart of the observation in Appendix~\ref{app:b0-isolation} that
a background change requested after the boundary must be expressed by the
suffix. Because $\phi$ is shared, the lemma does not imply that block $k$'s
output is unchanged after an update, so same-prompt drift is measured
separately.

\subsection{Why the objective is frozen before training}
The scales $s_r$ are measured once before training and then held fixed
(Appendix~\ref{app:current-implementation}).
\begin{proposition}[Adaptive normalization admits no fixed objective]
\label{prop:adaptive-normalization}
If term $r$ is reported as $\ell_r^{(t)}/s_r^{(t)}$ with $s_r^{(t)}>0$ a
statistic of the optimization history, there exist sequences in which
$\ell_r^{(t)}$ strictly increases while the reported value strictly decreases
to $0$. A monotone decrease of the reported value then certifies nothing
about the modelled quantity.
\end{proposition}
\begin{proof}
Take $\ell_r^{(t)}=t$ and $s_r^{(t)}=t^2$. The raw sequence grows without
bound while the reported sequence $1/t$ decreases to $0$. Since the objective
changes whenever $s_r^{(t)}$ does, the updates are not gradient steps on one
fixed function.
\end{proof}
\begin{corollary}[Frozen scales preserve the ordering]
\label{cor:frozen-order}
If $s_r$ is a positive constant, the reported term decreases if and only if
the raw term decreases, and $\mathcal L_{\rm edit}$ is one fixed function of
$\phi$.
\end{corollary}
Raw training terms are therefore checked from the first checkpoint against
preregistered directions, and a run whose raw preservation rises is stopped.
The same reasoning explains why preservation is scored against the paired
counterfactual from the source program rather than against a photograph.
Against a photograph the term has a floor that no edit can remove, measured at
$0.1405$ before the encoder has moved and $0.150$ after $700$ updates, which
reflects the gap between a generator and a camera rather than edit damage.

\subsection{What the detached prefix does and does not protect}
\begin{proposition}[No reference-moving shortcut]
\label{prop:no-shortcut}
The source state $\bar x_b$, source image $I_s$, keep-mask $m$ and observer
text features enter $\mathcal L_{\rm edit}$ only through $\sg[\cdot]$, so no
gradient step can reduce the loss by moving the target it is scored against.
\end{proposition}
\begin{proof}
Immediate from $\partial\,\sg[u]/\partial\phi=0$. Preservation against
$\sg[I_s]$, the response hinges against $\sg[q(I_s)]$ and the caption term
against detached text features all have fixed second arguments.
\end{proof}
Three further shortcuts are excluded by construction. The keep-mask is not
re-derived from the edited image, which would let the optimizer shrink the
penalized region by hiding the subject. The response terms hinge on a gain
relative to the detached source, so inflating presence everywhere does not
help. When both ends of a relation resolve to the same detector anchor, the
pair is dropped and reported as \texttt{relation\_degenerate\_fraction}.

\subsection{Cost accounting}
\begin{lemma}[Suffix denoiser calls]
\label{lem:nfe}
With $N$ steps, classifier-free guidance as one conditional and one
unconditional pass, and a branch at step $b$, resuming a saved prefix costs
$2(N-b)$ denoiser evaluations. The FlowEdit adaptation evaluates guided
velocities for a source and a target condition and costs $4(N-b)$, twice as
many at every $b$.
\end{lemma}
\begin{proof}
Each remaining step performs one guided evaluation (two passes), or two for
FlowEdit (four passes).
\end{proof}
At $b\in\{0,6,9,15,21\}$ with $N=30$ this gives $60/48/42/30/18$ against
$120/96/84/60/36$. These are suffix costs only; source generation, text
encoding and VAE encode/decode are excluded. Native shared-state reuse has
the same $2(N-b)$ cost, and call counts do not translate directly into
wall-clock time.

\subsection{Scope}
These results are structural. They concern the objective, its gradient
support, its normalization and its arithmetic, and assert no convergence,
rate or stationary point. The map $\phi\mapsto\mathcal L_{\rm edit}$ runs
through a deep sampler, a decoder and a detector-based observer with clamps,
hinges and a hard mask, and is neither convex nor globally smooth. A
standalone arithmetic script released with the code checks every numeric bound, the counterexample and the call counts without executing
the model.

\FloatBarrier
\section{Comparison protocols and independent evaluation}
\label{app:comparative-evaluation}
\paragraph{Source access defines the comparison.}
Full regeneration and learned prefix editing use identical target conditions
and noise, and differ only in how much source history is retained. The native
suffix control starts from the same saved state and changes the conditioning
representation. Image-input methods receive the source PNG and native
captions, without the stored sampler state.

\begin{table}[!htbp]\centering\small
\caption{Roles of the comparison methods.}
\begin{tabular}{p{.25\linewidth}p{.28\linewidth}p{.36\linewidth}}\toprule
Method & Source access & Question\\\midrule
Full regeneration & Same initial noise & Value of retaining source history\\
Native suffix & Same saved state & Value of learned future conditions\\
Initialized compiler & Same saved state & Effect of adaptation\\
SDEdit--SANA adaptation & Source PNG, re-noising & Image-input denoising control\\
FlowEdit--SANA adaptation & Source PNG, source/target velocities & Image-input flow editing control\\
\bottomrule\end{tabular}\end{table}

\paragraph{Image-input adaptations.}
We implement SDEdit-style re-noising and the FlowEdit velocity-difference
update on the same SANA generator \citep{meng2021sdedit,kulikov2025flowedit}.
The source PNG is VAE-encoded deterministically, and native source and target
captions use guidance 4 with the same 30-step schedule. Start indices 9, 15
and 21 are all retained. FlowEdit uses one noise sample per update and no
terminal target-only interval ($n_{\rm avg}=1$, $n_{\rm min}=0$,
$n_{\rm max}=21,15,9$), with the reference implementation pinned to commit
\texttt{2620a9364f0f}. These adaptations are not reproductions of the
published SD3/FLUX benchmarks. All 216 source cases were edited by both
methods at all three strengths, giving 1296 edited images and 216 VAE
reconstructions, with no output filtering. Call counts follow
Lemma~\ref{lem:nfe}.

\paragraph{Independent detector calibration.}
Two frozen category detectors, Grounding DINO tiny \citep{liu2023groundingdino} and an OWLv2 base
ensemble \citep{minderer2023owlv2}, are distinct from the training YOLO observer. On COCO val2017,
image identities are hash-split into 1327 fit and 1352 audit queries over 23
categories, with up to 30 positives and 30 negatives per category and split.
The global threshold maximizes fit localization F1 subject to precision at
least 0.90. A positive requires a top detection with IoU at least 0.5 against
an annotated object, and a wrongly localized positive counts as both a false
positive and a false negative.
\begin{table}[!htbp]\centering\small
\caption{Independent detector calibration on public images, not
generated-editing accuracy.}\label{tab:detector-calibration}
\begin{tabular}{lrrr}\toprule
Detector & Threshold & Audit precision & Audit recall\\\midrule
Grounding DINO tiny &0.40&0.886&0.819\\
OWLv2 base ensemble &0.31&0.900&0.861\\
\bottomrule\end{tabular}\end{table}
Both detectors fall short for banana, bench, bicycle, toaster, traffic light
and truck, and Grounding DINO also for cat and tennis racket. Calibration on
photographs does not establish accuracy on generated hybrids, attributes or
relations.

\paragraph{Generated-image category evidence.}
The image/query manifest was fixed before scoring and contains 2592 logical
records, giving 4752 raw calls per detector after deduplication. The two new
native schedules add 432 outputs and 1584 raw calls, reusing the frozen
source scores, and all new outputs are checksum-matched to the original images.
Source and target names that map to the same category, and out-of-vocabulary
replacements, are kept in the raw reports but not counted as validated
replacements. The measures are source detection, target detection, old
category absence and their replacement conjunction.

\begin{figure}[!tbp]\centering
\includegraphics[width=\linewidth]{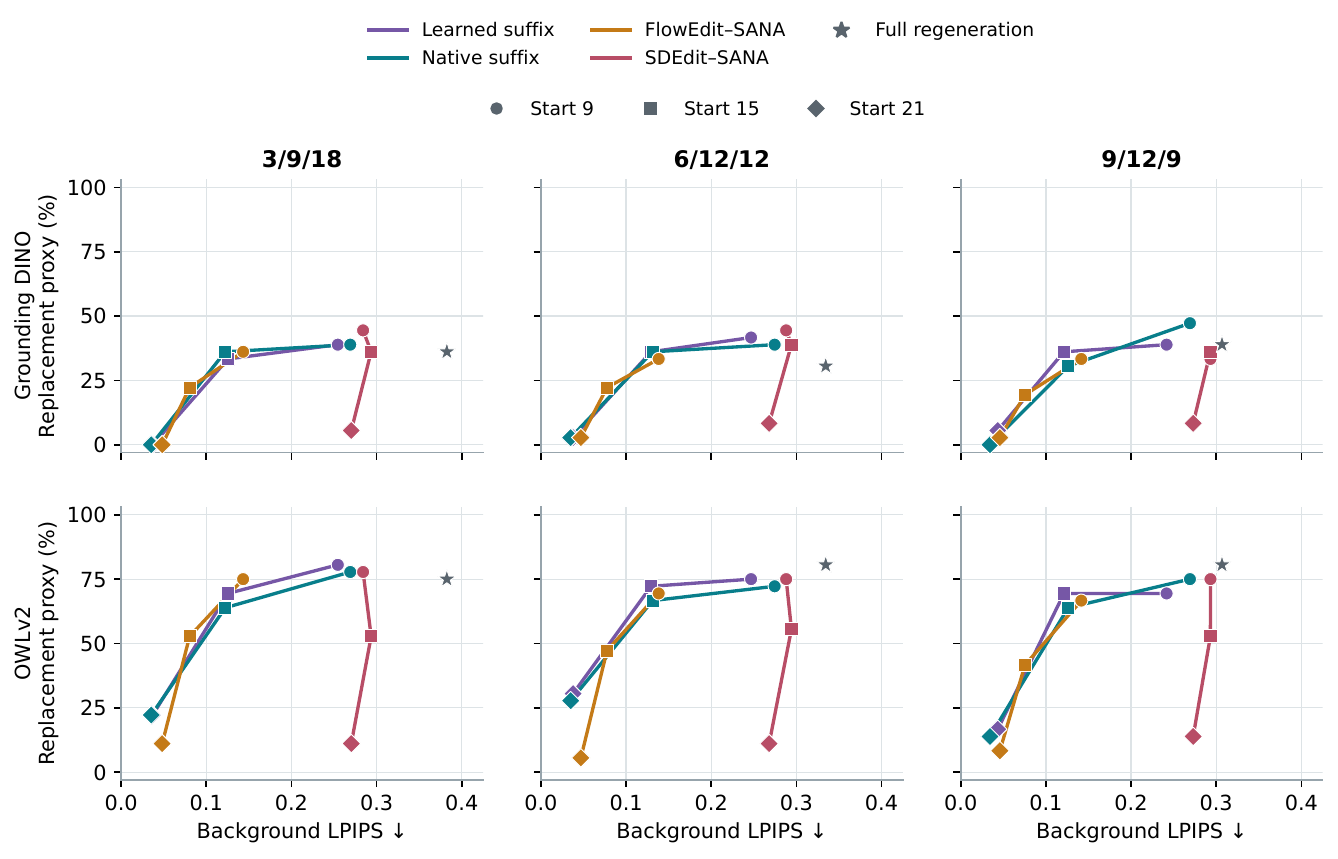}
\caption{\textbf{Preservation and category change on a common evaluation cohort.}
Each point uses the same 12 category families and three seeds. Replacement
proxies use all 36 attempts, and background LPIPS is defined for 32, 33 and
33 source masks across the schedules. Colors identify methods, marker shapes
identify starts 9, 15 and 21, and the star denotes full regeneration. Curves
connect descriptive means, not fitted frontiers, and the two detectors can
order methods differently.}
\label{fig:independent-tradeoff}
\end{figure}
\paragraph{Behavioral separation.}
We use separation in an operational sense, namely retaining one aspect of the
image while modifying another. At the same early saved state, learned
conditions change the background less than native conditioning
(Appendix~\ref{app:native-source-box-regions}), with similar category counts
in the first two schedules. This is a localized improvement in the
retention--change balance, not a uniformly better frontier, and it does not
identify encoder pretraining as the cause, since temporal responses also
occur at initialization.
\paragraph{Complete category and preservation comparison.}
The category cohort contains 12 families and three seeds (36 attempts per schedule) whose source and target categories are distinct and in the calibration vocabulary. The other 12 families (eight out of vocabulary, four within-category attribute edits) stay in the raw reports. Background LPIPS uses the unchanged historical source keep-masks and spatial AlexNet weights, recomputed on the same 36 attempts, with missing masks kept in the denominator. We report target detection (T) and the replacement conjunction (R), which requires the old category in the source, the target in the edit and the old category absent from the edit. Undetected source objects count as failures.
\par
\begin{table}[!htbp]\centering\small
\caption{Category proxies and preservation, schedule 3/9/18. T/R are counts out of 36 per detector; coverage is defined/36. Start is a saved-state branch for suffix methods and a re-noising or flow start for image-input methods.}\label{tab:category-3-9-18}
\begin{tabular}{lrrrrr}\toprule
Method & Start & BG LPIPS$\downarrow$ & Coverage & DINO T/R & OWLv2 T/R\\\midrule
Full regeneration & 0 & 0.383 & 32/36 & 35/13 & 36/27\\
Learned suffix & 9 & 0.254 & 32/36 & 36/14 & 36/29\\
Learned suffix & 15 & 0.125 & 32/36 & 34/12 & 35/25\\
Learned suffix & 21 & 0.037 & 32/36 & 30/0 & 17/8\\
Native suffix & 9 & 0.269 & 32/36 & 36/14 & 36/28\\
Native suffix & 15 & 0.122 & 32/36 & 34/13 & 35/23\\
Native suffix & 21 & 0.035 & 32/36 & 30/0 & 16/8\\
SDEdit--SANA & 9 & 0.284 & 32/36 & 35/16 & 36/28\\
SDEdit--SANA & 15 & 0.294 & 32/36 & 33/13 & 30/19\\
SDEdit--SANA & 21 & 0.270 & 32/36 & 27/2 & 13/4\\
FlowEdit--SANA & 9 & 0.143 & 32/36 & 34/13 & 35/27\\
FlowEdit--SANA & 15 & 0.080 & 32/36 & 31/8 & 32/19\\
FlowEdit--SANA & 21 & 0.048 & 32/36 & 24/0 & 15/4\\
\bottomrule\end{tabular}\end{table}
\begin{table}[!htbp]\centering\small
\caption{As Table~\ref{tab:category-3-9-18}, schedule 6/12/12.}\label{tab:category-6-12-12}
\begin{tabular}{lrrrrr}\toprule
Method & Start & BG LPIPS$\downarrow$ & Coverage & DINO T/R & OWLv2 T/R\\\midrule
Full regeneration & 0 & 0.334 & 33/36 & 34/11 & 35/29\\
Learned suffix & 9 & 0.247 & 33/36 & 35/15 & 36/27\\
Learned suffix & 15 & 0.129 & 33/36 & 34/13 & 33/26\\
Learned suffix & 21 & 0.038 & 33/36 & 31/1 & 23/11\\
Native suffix & 9 & 0.274 & 33/36 & 35/14 & 36/26\\
Native suffix & 15 & 0.131 & 33/36 & 35/13 & 33/24\\
Native suffix & 21 & 0.035 & 33/36 & 31/1 & 21/10\\
SDEdit--SANA & 9 & 0.288 & 33/36 & 35/16 & 36/27\\
SDEdit--SANA & 15 & 0.294 & 33/36 & 34/14 & 30/20\\
SDEdit--SANA & 21 & 0.268 & 33/36 & 29/3 & 19/4\\
FlowEdit--SANA & 9 & 0.138 & 33/36 & 33/12 & 34/25\\
FlowEdit--SANA & 15 & 0.078 & 33/36 & 33/8 & 30/17\\
FlowEdit--SANA & 21 & 0.047 & 33/36 & 27/1 & 18/2\\
\bottomrule\end{tabular}\end{table}
\begin{table}[!htbp]\centering\small
\caption{As Table~\ref{tab:category-3-9-18}, schedule 9/12/9.}\label{tab:category-9-12-9}
\begin{tabular}{lrrrrr}\toprule
Method & Start & BG LPIPS$\downarrow$ & Coverage & DINO T/R & OWLv2 T/R\\\midrule
Full regeneration & 0 & 0.307 & 33/36 & 35/14 & 35/29\\
Learned suffix & 9 & 0.242 & 33/36 & 35/14 & 34/25\\
Learned suffix & 15 & 0.121 & 33/36 & 35/13 & 34/25\\
Learned suffix & 21 & 0.043 & 33/36 & 30/2 & 21/6\\
Native suffix & 9 & 0.269 & 33/36 & 36/17 & 36/27\\
Native suffix & 15 & 0.126 & 33/36 & 36/11 & 34/23\\
Native suffix & 21 & 0.034 & 33/36 & 31/0 & 19/5\\
SDEdit--SANA & 9 & 0.293 & 33/36 & 35/12 & 36/27\\
SDEdit--SANA & 15 & 0.293 & 33/36 & 32/13 & 29/19\\
SDEdit--SANA & 21 & 0.273 & 33/36 & 27/3 & 17/5\\
FlowEdit--SANA & 9 & 0.142 & 33/36 & 34/12 & 33/24\\
FlowEdit--SANA & 15 & 0.075 & 33/36 & 30/7 & 29/15\\
FlowEdit--SANA & 21 & 0.046 & 33/36 & 28/1 & 13/3\\
\bottomrule\end{tabular}\end{table}
\paragraph{Detector disagreement and paired uncertainty.}
Before editing, Grounding DINO already detects the target category in 12, 13 and 13 of the 36 sources across schedules, against 1, 0 and 0 for OWLv2, and misses the source category in 3, 3 and 4 cases that OWLv2 detects. We therefore report both detectors. At branch 9, the learned-minus-native replacement differences in percentage points, with descriptive 95\% family-bootstrap intervals over 10000 resamples of the 12 family means, are as follows. For 3/9/18, DINO $+0.00$ [$-8.33$, $+8.33$] and OWLv2 $+2.78$ [$-5.56$, $+11.11$]. For 6/12/12, DINO $+2.78$ [$-11.11$, $+19.44$] and OWLv2 $+2.78$ [$-5.56$, $+11.11$]. For 9/12/9, DINO $-8.33$ [$-22.22$, $+0.00$] and OWLv2 $-5.56$ [$-16.67$, $+5.56$]. The ordering changes across detector and schedule, and none of these intervals establishes noninferiority.
\paragraph{Image-input controls on all evaluation cases.}
\begin{table}[!htbp]\centering\small
\caption{Background LPIPS of image-input SANA adaptations over all 72 source cases per schedule, using the historical source masks.}
\begin{tabular}{llrrrr}\toprule
Schedule & Method & Start 9 & Start 15 & Start 21 & Coverage\\\midrule
3/9/18 & SDEdit--SANA & 0.265 & 0.288 & 0.269 & 63/72\\
3/9/18 & FlowEdit--SANA & 0.122 & 0.066 & 0.041 & 63/72\\
6/12/12 & SDEdit--SANA & 0.273 & 0.291 & 0.273 & 61/72\\
6/12/12 & FlowEdit--SANA & 0.122 & 0.067 & 0.042 & 61/72\\
9/12/9 & SDEdit--SANA & 0.282 & 0.296 & 0.282 & 60/72\\
9/12/9 & FlowEdit--SANA & 0.127 & 0.066 & 0.042 & 60/72\\
\bottomrule\end{tabular}\end{table}
VAE reconstruction alone gives background LPIPS 0.0315, 0.0325 and 0.0330 for 3/9/18, 6/12/12 and 9/12/9. At the same nominal start, FlowEdit often retains more background than the learned suffix, while intermediate learned suffixes can reach higher replacement counts, and SDEdit changes more background. These are tradeoffs across different source-access protocols, not an overall ranking.

\paragraph{Automated multimodal judgments.}
The pairwise evaluation is a local adaptation of the GEditBench v2
instruction, consistency and quality rubrics \citep{jiang2026geditv2}, with a
Qwen judge in place of PVC-Judge and GPT. We report both image orders and keep
order disagreement, ties and missing calls separate. The 10368 completed calls
show no consistent trained-compiler advantage and substantial order
disagreement (Appendix~\ref{app:v13-final-evaluation}).

\FloatBarrier

\part*{Part II\quad Controlled setting}
\section{Absolute region errors in the toy case}
\label{app:toy-absolute-mse}
Figure~\ref{fig:toy-absolute-mse} retains the absolute measurements underlying
the pointwise ratios in Figure~\ref{fig:toy-mse-trajectory}. Normalization is
performed separately at each noise level using the same single-expert
reference. No data points are removed or smoothed. The additional trajectory
frames below the absolute curves show the same selected condition.
\begin{figure}[htbp]
\centering
\includegraphics[width=\linewidth]{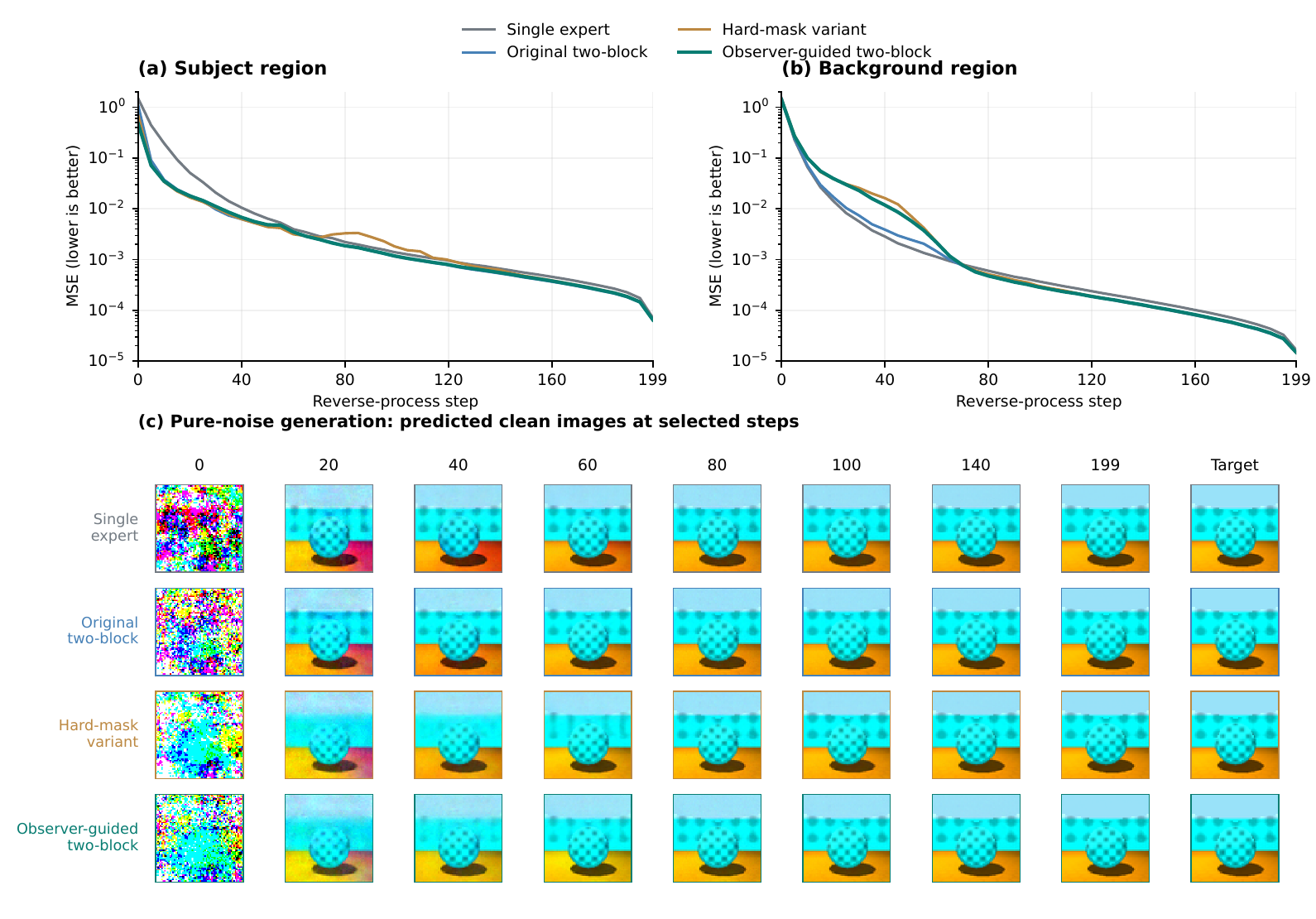}
\caption{Absolute subject and background MSE on logarithmic scales, before
normalization, with the fuller eight-step qualitative trajectory. The
quantitative noised-reference evaluation and the selected pure-noise
trajectory are separate experiments.}
\label{fig:toy-absolute-mse}
\end{figure}

\FloatBarrier

\part*{Part III\quad Editing with SANA}
\section{Final edit-only evaluation}
\label{app:v13-final-evaluation}
This section reports the three final checkpoints on the main evaluation set
and a challenge set.
\subsection{Objective measurements}
CLIP shares the training observer family, and background metrics are
conditional on accepted source masks. Each entry shows the defined count
against the attempted count, and family-bootstrap uncertainty is kept in the
archived summary.
\begin{table}[!htbp]
\centering\scriptsize\setlength{\tabcolsep}{3pt}
\caption{Final edit-only 24-family diagnostics: means (defined/attempted cases). Three seeds per request; step 0 is full regeneration. PSNR with zero MSE is undefined numerically and excluded from its mean.}\label{tab:v13-final-generation}
\begin{tabular}{lcccc}\toprule
Schedule; branch & CLIP $\Delta$ & BG LPIPS$\downarrow$ & BG PSNR$\uparrow$ & BG SSIM$\uparrow$\\\midrule
3/9/18; 0 & 0.097 (72/72) & 0.352 (63/72) & 15.652 (63/72) & 0.572 (63/72)\\
3/9/18; 9 & 0.092 (72/72) & 0.225 (63/72) & 19.073 (63/72) & 0.701 (63/72)\\
3/9/18; 15 & 0.076 (72/72) & 0.108 (63/72) & 24.316 (63/72) & 0.837 (63/72)\\
3/9/18; 21 & 0.028 (72/72) & 0.030 (63/72) & 31.710 (63/72) & 0.946 (63/72)\\
6/12/12; 0 & 0.092 (72/72) & 0.321 (61/72) & 16.500 (61/72) & 0.599 (61/72)\\
6/12/12; 9 & 0.088 (72/72) & 0.227 (61/72) & 19.096 (61/72) & 0.697 (61/72)\\
6/12/12; 15 & 0.076 (72/72) & 0.113 (61/72) & 24.322 (61/72) & 0.827 (61/72)\\
6/12/12; 21 & 0.029 (72/72) & 0.033 (61/72) & 31.156 (61/72) & 0.940 (61/72)\\
9/12/9; 0 & 0.086 (72/72) & 0.301 (60/72) & 17.407 (60/72) & 0.625 (60/72)\\
9/12/9; 9 & 0.086 (72/72) & 0.230 (60/72) & 19.477 (60/72) & 0.703 (60/72)\\
9/12/9; 15 & 0.073 (72/72) & 0.113 (60/72) & 24.517 (60/72) & 0.828 (60/72)\\
9/12/9; 21 & 0.028 (72/72) & 0.040 (60/72) & 30.372 (60/72) & 0.929 (60/72)\\
\bottomrule\end{tabular}
\end{table}
\begin{table}[!htbp]
\centering\scriptsize\setlength{\tabcolsep}{3pt}
\caption{As Table~\ref{tab:v13-final-generation}, for the 24-case challenge pool.}\label{tab:v13-final-challenge}
\begin{tabular}{lcccc}\toprule
Schedule; branch & CLIP $\Delta$ & BG LPIPS$\downarrow$ & BG PSNR$\uparrow$ & BG SSIM$\uparrow$\\\midrule
3/9/18; 0 & 0.049 (72/72) & 0.415 (10/72) & 15.553 (10/72) & 0.561 (10/72)\\
3/9/18; 9 & 0.042 (72/72) & 0.198 (10/72) & 21.826 (10/72) & 0.780 (10/72)\\
6/12/12; 0 & 0.048 (72/72) & 0.451 (11/72) & 15.372 (11/72) & 0.538 (11/72)\\
6/12/12; 9 & 0.041 (72/72) & 0.244 (11/72) & 20.267 (11/72) & 0.729 (11/72)\\
9/12/9; 0 & 0.047 (72/72) & 0.425 (12/72) & 16.233 (12/72) & 0.598 (12/72)\\
9/12/9; 9 & 0.042 (72/72) & 0.252 (12/72) & 20.327 (12/72) & 0.745 (12/72)\\
\bottomrule\end{tabular}
\end{table}
\begin{table}[!htbp]
\centering\scriptsize\setlength{\tabcolsep}{3pt}
\caption{Same-prompt drift at the final checkpoint: eight evaluation families, three seeds. No target edit is requested; these distances measure change from the frozen-source continuation.}\label{tab:v13-final-noop}
\begin{tabular}{lccc}\toprule
Schedule; branch & Whole LPIPS & PNG MSE & Latent MSE\\\midrule
3/9/18; 9 & 0.050 (24/24) & 0.003 (24/24) & 0.020 (24/24)\\
3/9/18; 15 & 0.018 (24/24) & 0.001 (24/24) & 0.004 (24/24)\\
3/9/18; 21 & 0.005 (24/24) & 0.000 (24/24) & 0.000 (24/24)\\
6/12/12; 9 & 0.057 (24/24) & 0.003 (24/24) & 0.019 (24/24)\\
6/12/12; 15 & 0.031 (24/24) & 0.001 (24/24) & 0.006 (24/24)\\
6/12/12; 21 & 0.012 (24/24) & 0.000 (24/24) & 0.001 (24/24)\\
9/12/9; 9 & 0.088 (24/24) & 0.005 (24/24) & 0.037 (24/24)\\
9/12/9; 15 & 0.037 (24/24) & 0.002 (24/24) & 0.008 (24/24)\\
9/12/9; 21 & 0.026 (24/24) & 0.001 (24/24) & 0.004 (24/24)\\
\bottomrule\end{tabular}
\end{table}
\subsection{Automated pairwise judgments}
The local GEditBench-style judge scores instruction following (IF), visual consistency (VC) and visual quality (VQ) in both image orders. These are automated judgments, not human preferences, and order disagreement is reported separately from ties. A staged-text suffix control was also judged for 9/12/9, but it differs from the native suffix only in B0, which has finished executing at every reported branch. Its 216 outputs are byte-identical to the native-suffix outputs, so it is omitted. Repeated judge calls on these identical pairs differ in a single branch-21 consistency outcome.
\begin{table}[!htbp]
\centering\scriptsize\setlength{\tabcolsep}{3pt}
\caption{Final 3/9/18 automated comparison: percentages win/loss/order-disagreement/tie/missing for the final checkpoint versus each control. Each cell has 72 planned pairs, 24 families and both image orders. Missing judgments are retained in the denominator.}\label{tab:v13-final-vlm-3-9-18}
\begin{tabular}{lcccc}\toprule
Branch & Control & IF & VC & VQ\\\midrule
9 & Initial & 11.1/6.9/22.2/59.7/0.0 & 9.7/12.5/34.7/43.1/0.0 & 9.7/22.2/41.7/26.4/0.0\\
9 & Update 600 & 9.7/9.7/11.1/69.4/0.0 & 8.3/8.3/25.0/58.3/0.0 & 11.1/12.5/44.4/31.9/0.0\\
15 & Initial & 0.0/5.6/13.9/80.6/0.0 & 2.8/12.5/19.4/65.3/0.0 & 4.2/9.7/20.8/65.3/0.0\\
15 & Update 600 & 0.0/5.6/9.7/84.7/0.0 & 0.0/5.6/18.1/76.4/0.0 & 2.8/8.3/18.1/70.8/0.0\\
21 & Initial & 0.0/0.0/0.0/100.0/0.0 & 0.0/0.0/0.0/100.0/0.0 & 0.0/0.0/0.0/100.0/0.0\\
21 & Update 600 & 0.0/0.0/0.0/100.0/0.0 & 0.0/0.0/1.4/98.6/0.0 & 0.0/0.0/0.0/100.0/0.0\\
\bottomrule\end{tabular}
\end{table}
\begin{table}[!htbp]
\centering\scriptsize\setlength{\tabcolsep}{3pt}
\caption{As Table~\ref{tab:v13-final-vlm-3-9-18}, schedule 6/12/12.}\label{tab:v13-final-vlm-6-12-12}
\begin{tabular}{lcccc}\toprule
Branch & Control & IF & VC & VQ\\\midrule
9 & Initial & 11.1/4.2/22.2/62.5/0.0 & 12.5/6.9/38.9/41.7/0.0 & 11.1/27.8/48.6/12.5/0.0\\
9 & Update 600 & 6.9/2.8/23.6/66.7/0.0 & 6.9/9.7/27.8/55.6/0.0 & 16.7/11.1/40.3/31.9/0.0\\
15 & Initial & 6.9/4.2/16.7/72.2/0.0 & 5.6/11.1/27.8/55.6/0.0 & 5.6/16.7/19.4/58.3/0.0\\
15 & Update 600 & 4.2/2.8/6.9/86.1/0.0 & 2.8/5.6/20.8/70.8/0.0 & 2.8/15.3/13.9/68.1/0.0\\
21 & Initial & 0.0/0.0/5.6/94.4/0.0 & 1.4/2.8/18.1/77.8/0.0 & 0.0/0.0/11.1/88.9/0.0\\
21 & Update 600 & 0.0/0.0/1.4/98.6/0.0 & 1.4/0.0/9.7/88.9/0.0 & 0.0/1.4/2.8/95.8/0.0\\
\bottomrule\end{tabular}
\end{table}
\begin{table}[!htbp]
\centering\scriptsize\setlength{\tabcolsep}{3pt}
\caption{As Table~\ref{tab:v13-final-vlm-3-9-18}, schedule 9/12/9.}\label{tab:v13-final-vlm-9-12-9}
\begin{tabular}{lcccc}\toprule
Branch & Control & IF & VC & VQ\\\midrule
9 & Initial & 5.6/22.2/26.4/45.8/0.0 & 11.1/26.4/36.1/26.4/0.0 & 16.7/36.1/40.3/6.9/0.0\\
9 & Update 600 & 6.9/9.7/23.6/59.7/0.0 & 11.1/22.2/25.0/41.7/0.0 & 8.3/23.6/51.4/16.7/0.0\\
9 & Native suffix & 25.0/33.3/22.2/19.4/0.0 & 33.3/29.2/31.9/5.6/0.0 & 22.2/37.5/40.3/0.0/0.0\\
15 & Initial & 2.8/12.5/15.3/69.4/0.0 & 4.2/22.2/31.9/41.7/0.0 & 4.2/19.4/41.7/34.7/0.0\\
15 & Update 600 & 4.2/6.9/18.1/70.8/0.0 & 4.2/18.1/30.6/47.2/0.0 & 5.6/16.7/34.7/43.1/0.0\\
15 & Native suffix & 6.9/23.6/29.2/40.3/0.0 & 18.1/22.2/43.1/16.7/0.0 & 18.1/37.5/31.9/12.5/0.0\\
21 & Initial & 2.8/1.4/11.1/84.7/0.0 & 2.8/19.4/38.9/38.9/0.0 & 8.3/13.9/25.0/52.8/0.0\\
21 & Update 600 & 2.8/1.4/4.2/91.7/0.0 & 1.4/9.7/33.3/55.6/0.0 & 6.9/4.2/15.3/73.6/0.0\\
21 & Native suffix & 4.2/1.4/11.1/83.3/0.0 & 2.8/23.6/33.3/40.3/0.0 & 6.9/11.1/30.6/51.4/0.0\\
\bottomrule\end{tabular}
\end{table}

\FloatBarrier
\section{Paired uncertainty for prefix reuse}
\label{app:paired-prefix}
This post-hoc analysis uses all 24 evaluation families and their three seeds.
For each family we average the paired branch-9-minus-regeneration differences
over seeds, then bootstrap families with replacement (10,000 resamples, seed
20260924). Preservation uses seeds with a valid source mask in both arms, and
the semantic change uses all attempts.
\begin{table}[!htbp]
\centering\small
\caption{Paired branch-9-minus-regeneration differences with family-bootstrap 95\% intervals. Negative LPIPS differences indicate better preservation; negative semantic differences indicate a smaller CLIP margin change.}
\begin{tabular}{lccc}\toprule
Schedule & Valid masks & $\Delta$BG LPIPS [95\% CI] & $\Delta$semantic proxy [95\% CI]\\\midrule
3/9/18 & 63/72 & -0.1248 [-0.1588, -0.0936] & -0.0046 [-0.0086, -0.0010]\\
6/12/12 & 61/72 & -0.0924 [-0.1115, -0.0740] & -0.0037 [-0.0070, -0.0005]\\
9/12/9 & 60/72 & -0.0696 [-0.0948, -0.0458] & -0.0007 [-0.0039, 0.0029]\\
\bottomrule\end{tabular}\end{table}
Background measurements cover 22, 21 and 21 families, and semantic
measurements all 24. The table weights families equally, so it can differ
slightly from the valid-attempt means in the main text. The preservation
interval excludes zero for every schedule. The semantic proxy also decreases
for 3/9/18 and 6/12/12, while for 9/12/9 its interval includes zero, which
does not prove equal edit effectiveness. Moving the edit point changes which
history is retained and which future can still be redirected, and reduced
background change can also accompany a weaker subject edit.
\begin{figure}[!htbp]\centering
\includegraphics[width=\linewidth]{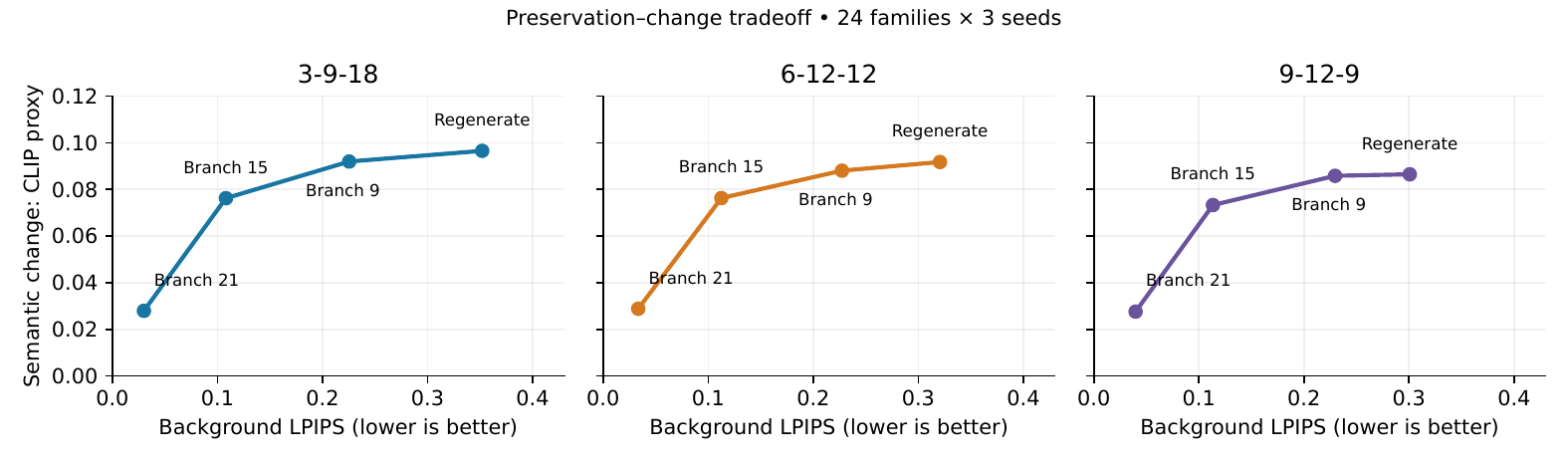}
\caption{All pre-specified branch points under the final edit-only checkpoints. Each curve shows a preservation--semantic-proxy tradeoff, not a validated success frontier. Background means are conditional on valid source masks; semantic means use all attempts.}\label{fig:development-tradeoff}
\end{figure}

\FloatBarrier
\section{Learned and native conditions at an identical prefix}
\label{app:native-suffix}
This comparison fixes the saved source state and changes only the suffix
conditions, using the final 9/12/9 compiler, 24 evaluation families, three
seeds and all three branch points. The native control repeats the native
target-caption features across the suffix, and source-image and prefix hashes
match within every pair. Because the common prefix comes from the frozen
initial compiler, this is an intervention on conditions at a matched state,
not a comparison of complete native and compiled trajectories.

\begin{table}[!htbp]\centering\small
\caption{Matched-prefix preservation at the final checkpoint. Background LPIPS uses
historical source masks: 60/72 cases from 21 families per branch. Lower is
better. The difference and interval give an equally weighted family estimate,
whereas the two means average valid attempts.}\label{tab:native-prefix-preservation}
\begin{tabular}{rrrrl}\toprule
Branch & Native & Trained & Lower in & $\Delta$ [descriptive 95\% interval]\\\midrule
9 &0.2528&0.2296&41/60&$-0.0216$ [$-0.0461$,$-0.0036$]\\
15&0.1134&0.1134&28/60&$+0.0000$ [$-0.0069$,$+0.0069$]\\
21&0.0305&0.0398&8/60&$+0.0090$ [$+0.0062$,$+0.0119$]\\
\bottomrule\end{tabular}\end{table}
The branch-9 valid-attempt mean decreases by 9.16\%, branch 15 is
essentially unchanged and branch 21 worsens. Intervals bootstrap the 21
family means (10000 repetitions, seed 20260924). The masks overlap the
training observer, and enlarged or moved targets can enter their nominal
background, so this does not establish equal editing effectiveness.

\begin{figure}[!htbp]\centering\setlength{\tabcolsep}{2pt}
\begin{tabular}{ccc}
\small Source&\small Native suffix&\small Trained suffix\\
\multicolumn{3}{l}{\small Bicycle to motorcycle; seed 303, branch 9}\\
\includegraphics[width=.31\linewidth]{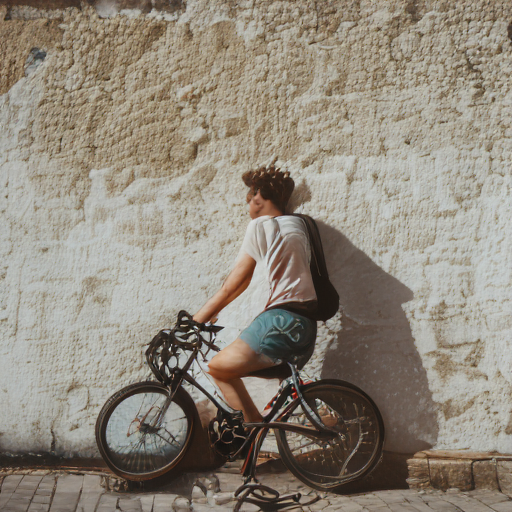}&
\includegraphics[width=.31\linewidth]{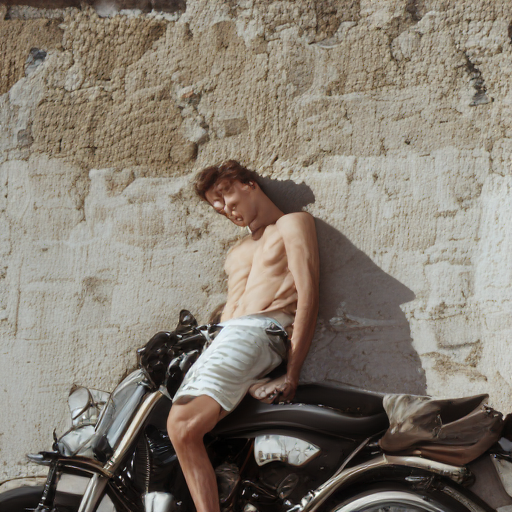}&
\includegraphics[width=.31\linewidth]{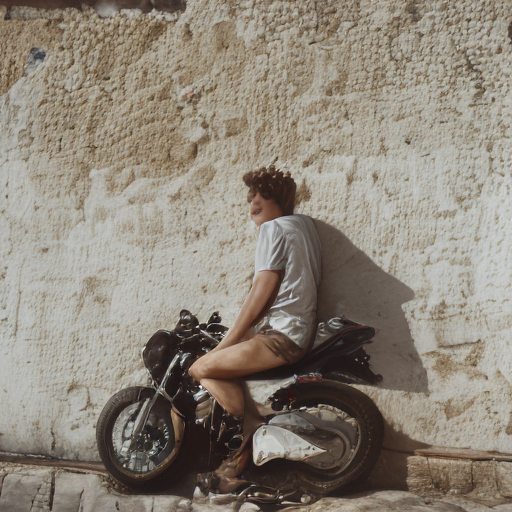}\\
&\scriptsize BG LPIPS 0.208&\scriptsize BG LPIPS 0.094\\[4pt]
\multicolumn{3}{l}{\small Older man to young woman; seed 101, branch 9}\\
\includegraphics[width=.31\linewidth]{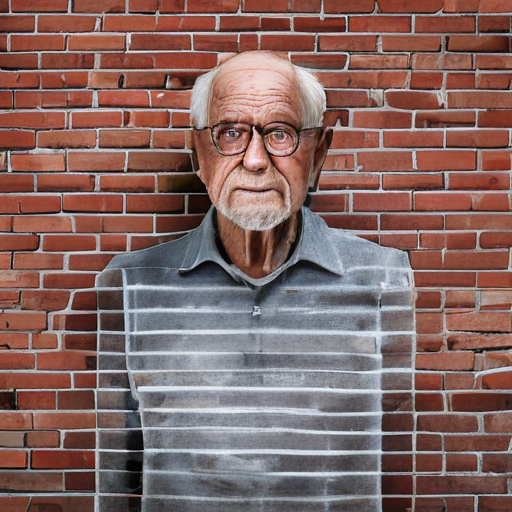}&
\includegraphics[width=.31\linewidth]{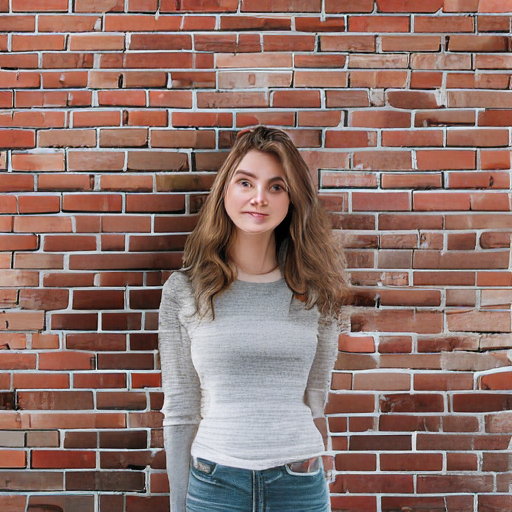}&
\includegraphics[width=.31\linewidth]{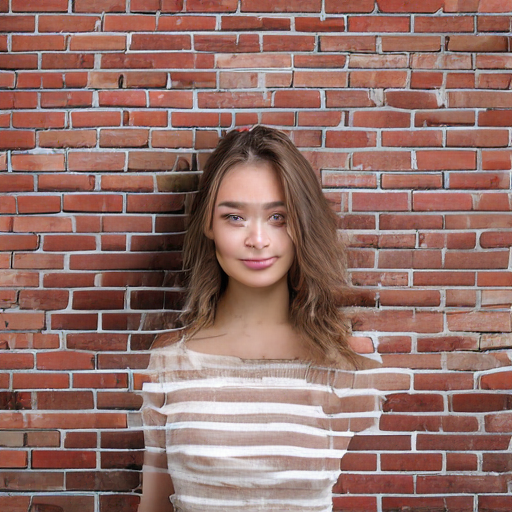}\\
&\scriptsize BG LPIPS 0.107&\scriptsize BG LPIPS 0.093\\
\end{tabular}
\caption{\textbf{Selected differences beyond prefix reuse.} Both suffixes
start from the same state. The trained output better retains the person's
scale and the wall in the first row, and the striped clothing and portrait
scale in the second. The bicycle caption also requests napping, which
neither edit realizes. Post-hoc examples from the final 9/12/9 compiler,
uncropped and unretouched.}\label{fig:native-suffix-examples}
\end{figure}
\paragraph{Inspection coverage and counterexamples.}
All 72 source cases and 216 branch comparisons were inspected.
Hydrant-to-lamp and bench-to-planter edits often behave similarly under both
methods. In an elephant-to-giraffe case both outputs alter the trees despite
lower LPIPS for the trained output, and the target caption still contains
``trunk''. Other examples preserve the source by failing a requested action
or omitting a requested phone.
\clearpage

\FloatBarrier
\section{Observer-independent region comparisons}
\label{app:observer-independent-regions}
The training preservation metric shares its observer with the compiler.
This appendix therefore rescores the same outputs on regions defined only by a
frozen external detector.

\subsection{Independent source-box regions}\label{app:independent-box-regions}
\paragraph{Protocol.} We score the complement of the union of source-only Grounding DINO boxes at the frozen 0.40 threshold, with 0, 5 or 10\% padding per box, and use the same region for every method editing a given source. Spatial AlexNet LPIPS is averaged over the defined background pixels. Attempts without a usable box stay in the attempted count. The full cohort has 72 attempts per schedule (24 families, three seeds) and the category-eligible cohort 36 (12 families). Differences are paired by family, seed, schedule and source, and intervals resample families 10,000 times (seed 20260924).

\begin{table}[!htbp]
\centering
\caption{Source-box background LPIPS on the full evaluation cohort. Entries are paired mean differences (ours branch 9 minus comparator; lower is better). $n$ is defined/attempted, with missing cases retained.}
\label{tab:independent-source-box-full}
\small\setlength{\tabcolsep}{4pt}
\begin{tabular}{lcccccc}
\toprule
Schedule & pad & $n$ & Regen. & FlowEdit & SDEdit & Native \\
\midrule
3-9-18 & 0\% & 68/72 & $-0.122$ & $+0.100$ & $-0.040$ & $-0.016$ \\
3-9-18 & 5\% & 65/72 & $-0.129$ & $+0.097$ & $-0.052$ & $-0.017$ \\
3-9-18 & 10\% & 62/72 & $-0.127$ & $+0.097$ & $-0.056$ & $-0.016$ \\
\addlinespace
6-12-12 & 0\% & 66/72 & $-0.092$ & $+0.105$ & $-0.043$ & $-0.024$ \\
6-12-12 & 5\% & 62/72 & $-0.094$ & $+0.104$ & $-0.055$ & $-0.022$ \\
6-12-12 & 10\% & 60/72 & $-0.090$ & $+0.100$ & $-0.064$ & $-0.023$ \\
\addlinespace
9-12-9 & 0\% & 64/72 & $-0.070$ & $+0.098$ & $-0.055$ & $-0.030$ \\
9-12-9 & 5\% & 62/72 & $-0.068$ & $+0.096$ & $-0.067$ & $-0.030$ \\
9-12-9 & 10\% & 60/72 & $-0.066$ & $+0.092$ & $-0.075$ & $-0.027$ \\
\bottomrule
\end{tabular}
\end{table}

\begin{table}[!htbp]
\centering
\caption{As Table~\ref{tab:independent-source-box-full}, on the eligible cohort.}
\label{tab:independent-source-box-eligible}
\small\setlength{\tabcolsep}{4pt}
\begin{tabular}{lcccccc}
\toprule
Schedule & pad & $n$ & Regen. & FlowEdit & SDEdit & Native \\
\midrule
3-9-18 & 0\% & 33/36 & $-0.114$ & $+0.114$ & $-0.023$ & $-0.011$ \\
3-9-18 & 5\% & 30/36 & $-0.124$ & $+0.111$ & $-0.043$ & $-0.014$ \\
3-9-18 & 10\% & 30/36 & $-0.127$ & $+0.107$ & $-0.048$ & $-0.012$ \\
\addlinespace
6-12-12 & 0\% & 33/36 & $-0.083$ & $+0.109$ & $-0.035$ & $-0.029$ \\
6-12-12 & 5\% & 30/36 & $-0.083$ & $+0.106$ & $-0.058$ & $-0.031$ \\
6-12-12 & 10\% & 30/36 & $-0.083$ & $+0.097$ & $-0.070$ & $-0.032$ \\
\addlinespace
9-12-9 & 0\% & 32/36 & $-0.062$ & $+0.094$ & $-0.050$ & $-0.031$ \\
9-12-9 & 5\% & 30/36 & $-0.061$ & $+0.089$ & $-0.067$ & $-0.031$ \\
9-12-9 & 10\% & 29/36 & $-0.061$ & $+0.082$ & $-0.082$ & $-0.031$ \\
\bottomrule
\end{tabular}
\end{table}

\begin{table}[!htbp]
\centering
\caption{Selected 95\% family-bootstrap intervals for the full evaluation cohort. All comparator intervals and RGB MSE are available in the analysis artifact.}
\label{tab:independent-source-box-intervals}
\small\setlength{\tabcolsep}{4pt}
\begin{tabular}{llccc}
\toprule
Comparator & Schedule & pad & $n$ & $\Delta$ [95\% interval] \\
\midrule
Regeneration & 3-9-18 & 0\% & 68/72 & $-0.122\,[-0.155,-0.092]$ \\
Regeneration & 3-9-18 & 5\% & 65/72 & $-0.129\,[-0.168,-0.095]$ \\
Regeneration & 3-9-18 & 10\% & 62/72 & $-0.127\,[-0.170,-0.090]$ \\
\addlinespace
Regeneration & 6-12-12 & 0\% & 66/72 & $-0.092\,[-0.113,-0.073]$ \\
Regeneration & 6-12-12 & 5\% & 62/72 & $-0.094\,[-0.116,-0.072]$ \\
Regeneration & 6-12-12 & 10\% & 60/72 & $-0.090\,[-0.110,-0.070]$ \\
\addlinespace
Regeneration & 9-12-9 & 0\% & 64/72 & $-0.070\,[-0.093,-0.048]$ \\
Regeneration & 9-12-9 & 5\% & 62/72 & $-0.068\,[-0.091,-0.047]$ \\
Regeneration & 9-12-9 & 10\% & 60/72 & $-0.066\,[-0.091,-0.043]$ \\
\midrule
Native shared & 9-12-9 & 0\% & 64/72 & $-0.030\,[-0.054,-0.012]$ \\
Native shared & 9-12-9 & 5\% & 62/72 & $-0.030\,[-0.057,-0.009]$ \\
Native shared & 9-12-9 & 10\% & 60/72 & $-0.027\,[-0.054,-0.006]$ \\
\bottomrule
\end{tabular}
\end{table}

\paragraph{Summary.} For every schedule and padding, ours at branch 9 has lower source-box background LPIPS than full regeneration and higher LPIPS than FlowEdit--SANA. The SDEdit--SANA and native comparisons depend more on cohort and padding, and the eligible 9/12/9 native intervals at 5 and 10\% include zero. The metric does not account for missed source boxes or for edit success.

\FloatBarrier
\subsection{Native shared-path comparison across branch points}\label{app:native-source-box-regions}
The protocol follows Appendix~\ref{app:independent-box-regions}. Each entry is the paired mean difference in source-box background LPIPS, ours minus native shared path, so negative values favor ours. Native outputs for 3/9/18 and 6/12/12 are new, and 9/12/9 reuses the existing outputs.

\begin{table}[!htbp]
\centering
\caption{Native shared-path source-box LPIPS comparison, full evaluation cohort. Point differences are ours minus native; $n$ is defined/attempted.}
\label{tab:native-source-box-full}
\small\setlength{\tabcolsep}{4pt}
\begin{tabular}{llccc}
\toprule
Schedule & Branch & pad 0\% & pad 5\% & pad 10\% \\
\midrule
3-9-18 & 9 & $-0.016$ (68/72) & $-0.017$ (65/72) & $-0.016$ (62/72) \\
3-9-18 & 15 & $+0.003$ (68/72) & $+0.003$ (65/72) & $+0.003$ (62/72) \\
3-9-18 & 21 & $+0.001$ (68/72) & $+0.001$ (65/72) & $+0.001$ (62/72) \\
\addlinespace
6-12-12 & 9 & $-0.024$ (66/72) & $-0.022$ (62/72) & $-0.023$ (60/72) \\
6-12-12 & 15 & $-0.002$ (66/72) & $-0.001$ (62/72) & $-0.000$ (60/72) \\
6-12-12 & 21 & $+0.003$ (66/72) & $+0.003$ (62/72) & $+0.003$ (60/72) \\
\addlinespace
9-12-9 & 9 & $-0.030$ (64/72) & $-0.030$ (62/72) & $-0.027$ (60/72) \\
9-12-9 & 15 & $-0.005$ (64/72) & $-0.004$ (62/72) & $-0.002$ (60/72) \\
9-12-9 & 21 & $+0.008$ (64/72) & $+0.009$ (62/72) & $+0.009$ (60/72) \\
\bottomrule
\end{tabular}
\end{table}

\begin{table}[!htbp]
\centering
\caption{As Table~\ref{tab:native-source-box-full}, on the eligible cohort.}
\label{tab:native-source-box-eligible}
\small\setlength{\tabcolsep}{4pt}
\begin{tabular}{llccc}
\toprule
Schedule & Branch & pad 0\% & pad 5\% & pad 10\% \\
\midrule
3-9-18 & 9 & $-0.011$ (33/36) & $-0.014$ (30/36) & $-0.012$ (30/36) \\
3-9-18 & 15 & $+0.003$ (33/36) & $+0.003$ (30/36) & $+0.003$ (30/36) \\
3-9-18 & 21 & $+0.002$ (33/36) & $+0.002$ (30/36) & $+0.002$ (30/36) \\
\addlinespace
6-12-12 & 9 & $-0.029$ (33/36) & $-0.031$ (30/36) & $-0.032$ (30/36) \\
6-12-12 & 15 & $-0.002$ (33/36) & $-0.002$ (30/36) & $-0.002$ (30/36) \\
6-12-12 & 21 & $+0.003$ (33/36) & $+0.003$ (30/36) & $+0.003$ (30/36) \\
\addlinespace
9-12-9 & 9 & $-0.031$ (32/36) & $-0.031$ (30/36) & $-0.031$ (29/36) \\
9-12-9 & 15 & $-0.008$ (32/36) & $-0.007$ (30/36) & $-0.006$ (29/36) \\
9-12-9 & 21 & $+0.008$ (32/36) & $+0.009$ (30/36) & $+0.010$ (29/36) \\
\bottomrule
\end{tabular}
\end{table}

At branch 9 with 5\% padding on the full cohort, the paired intervals are
$[-0.0407,-0.0016]$, $[-0.0397,-0.0081]$ and $[-0.0567,-0.0087]$ for 3/9/18,
6/12/12 and 9/12/9. Every branch point is shown so that branch 9 is not
selected after the fact; the remaining intervals are in the analysis
artifact.

\FloatBarrier
\section{Cost of reusing a common prefix}
\label{app:prefix-cost}
We compare caching one branch-9 state with recomputing the same prefix before each target continuation, using four scenes, four target prompts per scene, one noise seed and three timing repetitions on an NVIDIA A800-SXM4-80GB. Both modes include condition encoding, prefix and suffix execution and decoding, and exclude model loading and file writes. Mode order alternates across repetitions.
\begin{table}[!htbp]\centering\small
\caption{Mean time over four scenes and three repetitions to produce four target continuations. Every cached endpoint matches its recomputed counterpart exactly.}
\begin{tabular}{lrrrr}\toprule
Schedule & Recomputed (s) & Cached (s) & Speedup & Time saved\\\midrule
3/9/18 & 9.28 & 7.19 & 1.29$\times$ & 22.6\%\\
6/12/12 & 9.20 & 7.18 & 1.28$\times$ & 22.0\%\\
9/12/9 & 9.39 & 7.19 & 1.31$\times$ & 23.4\%\\
\bottomrule\end{tabular}\end{table}
Each saved latent occupies 32\,KiB. For $m$ continuations with $N=30$ and $b=9$, the two modes cost $mN$ and $b+m(N-b)$ denoising steps, so four continuations need 93 steps instead of 120 (22.5\% fewer). All 144 paired endpoints (three compilers, four scenes, three repetitions, four targets) pass the equality check. Native prompt switching can cache a prefix in the same way, so the saving is a property of the interface rather than of training.

\FloatBarrier

\part*{Part IV\quad Mechanism and interventions}
\section{Schedule, source-path and block interventions}
\label{app:mechanism-current}
\paragraph{Mechanism tests at the final edit-only checkpoints.}
\label{sec:v13-mechanism}
We ask whether training changes the response to future conditions, and
whether that response depends on the execution schedule or the source path.
We use the final checkpoints and the first eight families of the
evaluation set with its three seeds.

\paragraph{Training schedule versus inference schedule.}
A $3\times3$ crossing applies each trained compiler, and its
initialization, under all three inference schedules. Within each schedule,
all methods resume identical frozen-source prefixes at steps 0, 9, 15 and 21,
where step 0 is full regeneration. Each compiler keeps its original encoder
clock, so off-diagonal cells test a mismatch between execution and encoder
timing. Without repeated training seeds, differences between checkpoints do
not establish a causal schedule effect.
\begin{table}[!htbp]
\centering\scriptsize\setlength{\tabcolsep}{3pt}
\caption{Complete crossing, shown at branch 15. Cells report CLIP paired margin change / background LPIPS, with defined denominators in parentheses (24 attempted). These proxies do not establish edit success.}\label{tab:v13-crossing}
\begin{tabular}{lccc}\toprule
Training schedule & 3/9/18 & 6/12/12 & 9/12/9\\\midrule
3/9/18 & 0.074 (24) / 0.129 (24) & 0.070 (24) / 0.133 (23) & 0.075 (24) / 0.139 (22)\\
6/12/12 & 0.076 (24) / 0.138 (24) & 0.073 (24) / 0.138 (23) & 0.078 (24) / 0.144 (22)\\
9/12/9 & 0.077 (24) / 0.160 (24) & 0.070 (24) / 0.144 (23) & 0.074 (24) / 0.143 (22)\\
Initialization & 0.072 (24) / 0.126 (24) & 0.073 (24) / 0.134 (23) & 0.076 (24) / 0.143 (22)\\
\bottomrule\end{tabular}
\end{table}
\begin{table}[!htbp]
\centering\scriptsize\setlength{\tabcolsep}{3pt}
\caption{Diagonal source-path comparison. CLIP paired margin change / background LPIPS (defined denominators); different source modes use different source images.}\label{tab:v13-own-source}
\begin{tabular}{lcc}\toprule
Schedule; branch & Matched source & Own source\\\midrule
3/9/18; 0 & 0.101 (24) / 0.399 (24) & 0.101 (24) / 0.390 (22)\\
3/9/18; 9 & 0.091 (24) / 0.266 (24) & 0.092 (24) / 0.254 (22)\\
3/9/18; 15 & 0.074 (24) / 0.129 (24) & 0.074 (24) / 0.126 (22)\\
3/9/18; 21 & 0.019 (24) / 0.036 (24) & 0.022 (24) / 0.038 (22)\\
6/12/12; 0 & 0.095 (24) / 0.375 (23) & 0.095 (24) / 0.369 (22)\\
6/12/12; 9 & 0.089 (24) / 0.272 (23) & 0.088 (24) / 0.260 (22)\\
6/12/12; 15 & 0.073 (24) / 0.138 (23) & 0.077 (24) / 0.131 (22)\\
6/12/12; 21 & 0.027 (24) / 0.042 (23) & 0.024 (24) / 0.035 (22)\\
9/12/9; 0 & 0.093 (24) / 0.336 (22) & 0.093 (24) / 0.331 (21)\\
9/12/9; 9 & 0.091 (24) / 0.264 (22) & 0.088 (24) / 0.248 (21)\\
9/12/9; 15 & 0.074 (24) / 0.143 (22) & 0.071 (24) / 0.124 (21)\\
9/12/9; 21 & 0.025 (24) / 0.049 (22) & 0.027 (24) / 0.034 (21)\\
\bottomrule\end{tabular}
\end{table}
\begin{table}[!htbp]
\centering\scriptsize\setlength{\tabcolsep}{3pt}
\caption{Block interventions at branch 0, where all blocks are active. CLIP paired margin change / background LPIPS (defined denominators). Initial and trained methods share the frozen source; these are descriptive measurements.}\label{tab:v13-block-interventions}
\begin{tabular}{lccccc}\toprule
Compiler & B0 & B1 & B2 & all & nochange\\\midrule
Trained 3/9/18 & \shortstack{0.001 (24)\\0.179 (24)} & \shortstack{0.015 (24)\\0.306 (24)} & \shortstack{0.090 (24)\\0.237 (24)} & \shortstack{0.101 (24)\\0.399 (24)} & \shortstack{-0.000 (24)\\0.147 (24)}\\
Trained 6/12/12 & \shortstack{0.001 (24)\\0.204 (23)} & \shortstack{0.040 (24)\\0.301 (23)} & \shortstack{0.048 (24)\\0.114 (23)} & \shortstack{0.095 (24)\\0.375 (23)} & \shortstack{0.000 (24)\\0.070 (23)}\\
Trained 9/12/9 & \shortstack{0.004 (24)\\0.221 (22)} & \shortstack{0.059 (24)\\0.245 (22)} & \shortstack{0.027 (24)\\0.094 (22)} & \shortstack{0.093 (24)\\0.336 (22)} & \shortstack{0.000 (24)\\0.078 (22)}\\
Init. 3/9/18 & \shortstack{0.000 (24)\\0.088 (24)} & \shortstack{0.012 (24)\\0.319 (24)} & \shortstack{0.087 (24)\\0.191 (24)} & \shortstack{0.099 (24)\\0.411 (24)} & \shortstack{0.000 (24)\\0.000 (24)}\\
Init. 6/12/12 & \shortstack{-0.000 (24)\\0.190 (23)} & \shortstack{0.039 (24)\\0.312 (23)} & \shortstack{0.047 (24)\\0.078 (23)} & \shortstack{0.093 (24)\\0.382 (23)} & \shortstack{0.000 (24)\\0.000 (23)}\\
Init. 9/12/9 & \shortstack{0.006 (24)\\0.208 (22)} & \shortstack{0.060 (24)\\0.256 (22)} & \shortstack{0.026 (24)\\0.039 (22)} & \shortstack{0.094 (24)\\0.343 (22)} & \shortstack{0.000 (24)\\0.000 (22)}\\
\bottomrule\end{tabular}
\end{table}
Across all four branches, 1200 inactive interventions passed exact latent equality to their same-compiler source-condition continuation. This is a routing gate, not a learned-specialization result.

\paragraph{Matched source versus own source.}
Three further diagonal conditions generate the source with the trained
compiler used for editing. Matched sources isolate the suffix intervention
at a fixed state, whereas own sources test the complete path but also change
the source image and prefix. Background retention is measured against each
condition's actual source.

\paragraph{Block-specific interventions.}
Source and target conditions share a common sequence length with
zero-valued masked padding, and a padding-only control (``nochange'') records
the resulting numerical drift. For the initialization and the three trained
compilers, at a common saved prefix, we replace the source conditions by
target conditions in $B_0$, $B_1$, $B_2$, all blocks or none, at branches
0, 9, 15 and 21. Replacing a block whose execution has ended must have no
effect, which is a routing check rather than evidence of specialization.
Specialization would additionally require differentiated effects among active
blocks relative to initialization.

\FloatBarrier
\section{Stage readouts and single-block interventions}
\label{app:stage-factorization}
\paragraph{The model's own generation path.}
Here we generate from noise with the final 6/12/12 compiler itself, rather
than editing a frozen reference path. Eight fixed evaluation families at
close framing and seed 101 are compared with the role-fitted initialization.
Each model receives the source caption, subject-only and background-only
changes, and six single-block substitutions, for 144 trajectories and 576
archived images. The displayed case is the first family, not a selected
winner.

\begin{figure}[!htbp]
\centering
\includegraphics[width=\linewidth]{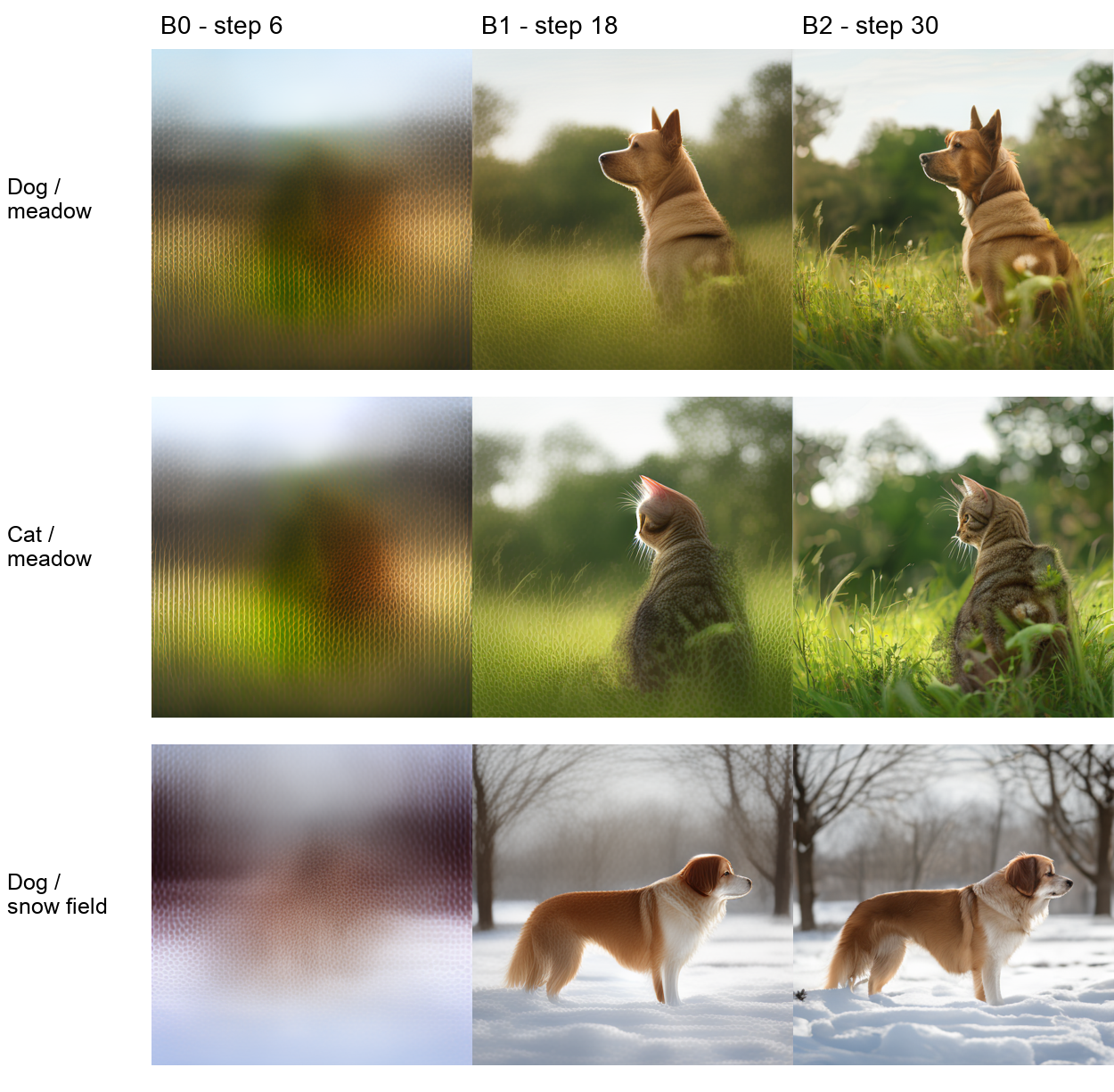}
\caption{\textbf{Generation along the learned path.} Final compiler, identical
initial noise, with source, subject-changed and background-changed captions.
Columns show $\widehat{x}_0=x_t-\sigma_t v_\theta(x_t,c_k)$ from the last
executed denoiser evaluation in each block (steps 6, 18, 30), not decodes of
noisy states. B0 exposes coarse scene appearance, B1 makes the subject
recognizable and B2 completes the image. Changing the scene also changes
pose and appearance, so stage progression alone is not factor independence.}
\label{fig:model-stage-readouts}
\end{figure}
\FloatBarrier

\paragraph{Which block controls the change?}
Each substitution replaces exactly one source block and its mask with those
compiled from the target caption, holding all other conditions, masks and
noise fixed. In the displayed case, replacing B0 with the cat-conditioned
block leaves the dog largely unchanged, whereas replacing B2 changes its
category. A scene-conditioned B0 changes the dog and composition without
producing snow, and replacing B1 produces snow but also changes pose. Later
blocks can therefore still determine scene content, and B2 can change
semantics rather than only texture.

\paragraph{Interpretation and limits.}
The eight-family gallery shows a useful temporal organization but not a
universal scene-only B0, since the zebra-on-beach caption already produces an
animal-shaped silhouette in its B0 readout, and single-block substitutions
can create hybrid anatomy such as elephant--giraffe mixtures.

Several features of the implementation let temporal organization and strict
independence differ. Role fitting masks a full-caption encoding whose
unmasked token features may retain subject context. The masks are cumulative
and the compiler shares a context trunk, so suffix training can change
early-block outputs. All blocks also act through the same spatial state and
denoiser. For fixed masks, with
$x_{t+1}=x_t+\Delta\sigma_t v(x_t,c_k)$, a condition perturbation propagates as
\begin{equation}
 \delta x_{t+1}=
 (I+\Delta\sigma_t D_xv)\,\delta x_t+
 \Delta\sigma_t D_cv\,\delta c_k
 \quad\text{to first order}.
\end{equation}
Temporal routing alone does not force the background projection of this
response to vanish, consistent with the criterion in
Appendix~\ref{app:factor-independence}. Source-box metrics add their own
coupling, as they count legitimate changes outside the old silhouette as
background damage, and foreground rectangles contain background pixels.

\FloatBarrier
\section{An explicit-background input-isolation variant}
\label{app:b0-isolation}
\paragraph{Intervention and its scope.}
The full-caption compiler masks contextual text features, so B0 can retain
subject information even when its direct token channel is masked. We
therefore test a training-free input intervention on the final 6/12/12
checkpoint. An explicit background sentence, taken from structured scenario
metadata and describing the scene and lighting, is encoded separately
through the text encoder, compiler and frozen mask predictor, and its B0
replaces the full-caption B0. B1 and B2 are unchanged and carry the subject,
framing and viewpoint instructions, so the early composition prior changes
as well. For a suffix resumed at step 6 or later the target B0 is no longer
executed, and a background change requested after that point must be
expressed by the suffix.

\paragraph{A precise invariance, rather than full-image independence.}
Write $c_0=F_0(E(b))$ and $m_0=M_0(E(b))$ for the separately compiled
background $b$, and let $s$ denote the subject description. For fixed model,
noise $z$ and deterministic sampler, the state at the B0 boundary satisfies
\begin{equation}
 x_6(b,s,z)=\Phi_0(z;F_0(E(b)),M_0(E(b)))=x_6(b,s',z).
 \label{eq:b0-input-invariance}
\end{equation}
This invariance to subject input follows from the routing, not from
training. It does not make foreground and background pixels independent
after B1/B2, which still act on a shared spatial state.

\paragraph{Initial mechanism check.}
All eight fixed evaluation families at seed 101 and close framing were
tested, giving 48 trajectories and 192 stage and endpoint images. Under
subject replacement, the isolated B0 condition, clean prediction and
post-block latent are exactly unchanged in 8/8 cases, against 0/8 for the
original compiler. Changing the background changes the isolated B0 state in
8/8 cases, which rules out a constant state. Under the full-caption compiler, by contrast, the
step-6 state already carries the replaced subject, with leak magnitudes from
0.081 to 0.510 in state and from 1.02 to 9.50 in the clean prediction. In the reviewed galleries the
target categories stay recognizable, but subject scale and composition
change.

\paragraph{Complete matched evaluation.}
Because the check above fixes one seed and framing, we also ran the full
frozen protocol on the same checkpoint, with 48 sources (8 scene families,
2 framings, 3 seeds), both edit directions, five starts $\{0,6,9,15,21\}$ and
six methods. This gives 2{,}880 outputs plus 96 source and reconstruction
images. Every source is regenerated by the B0-isolated compiler and all
controls are rerun on it, with detector thresholds, source-box masks and the
CLIP evaluator reused without retuning.

\begin{figure}[!htbp]
\centering
\includegraphics[width=.95\linewidth]{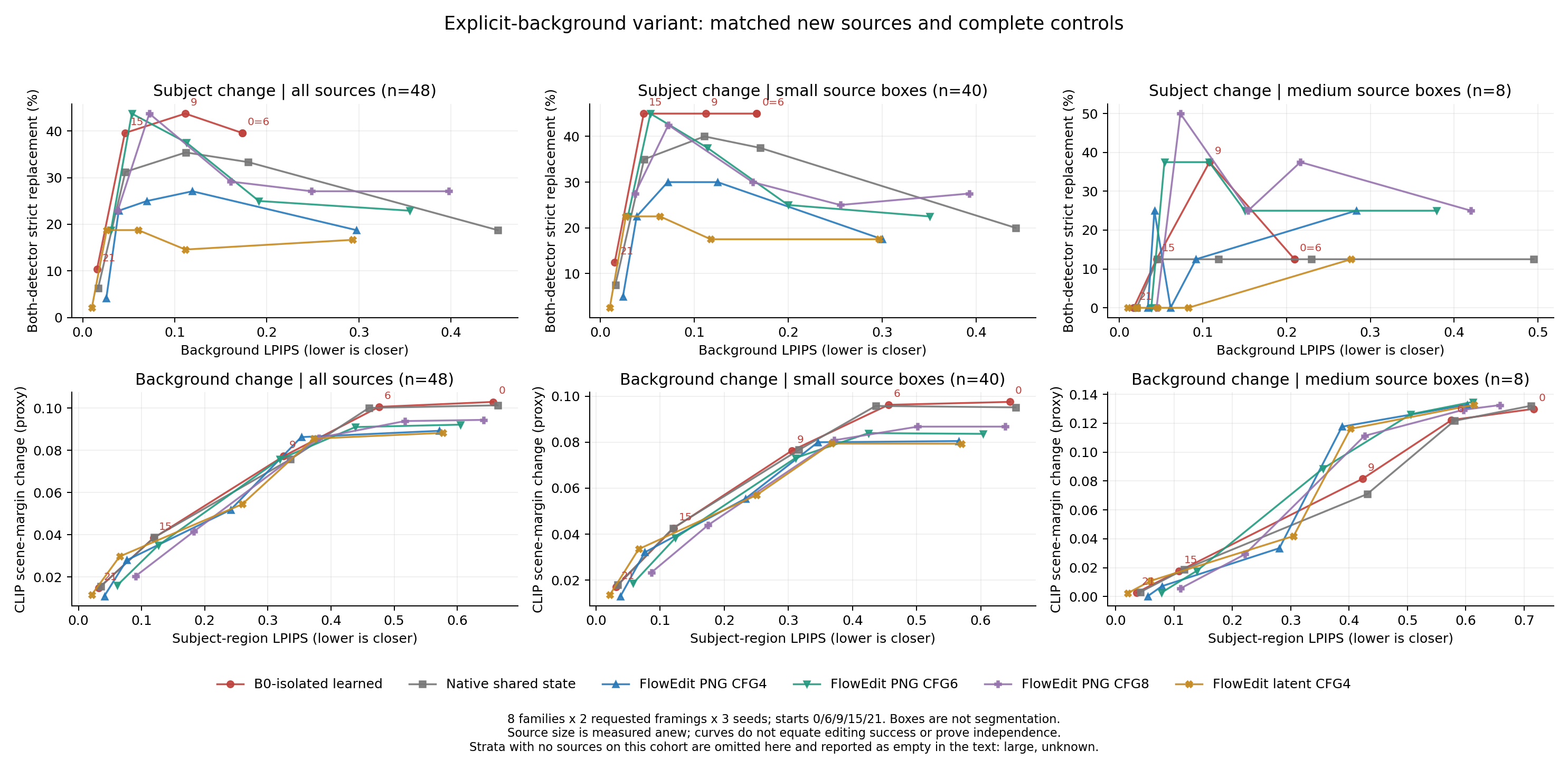}
\caption{\textbf{Edit strength against source closeness, all five starts.}
Each curve traces one method across starts 0, 6, 9, 15 and 21, for the
whole cohort and the two non-empty source-size strata. Later starts stay
closer to the source but edit less, so low LPIPS at a late start indicates a
weak edit. Subject panels use both-detector strict replacement, and
background panels the uncalibrated CLIP scene margin. The medium stratum
holds only eight sources.}
\label{fig:b0-isolated-tradeoff}
\end{figure}
\FloatBarrier

\paragraph{Matched subject-replacement comparison.}
Native suffixes receive exactly the same saved state, all five starts are retained, and source category failures stay in the denominator.
\begin{table}[!htbp]\centering\small
\setlength{\tabcolsep}{4pt}
\caption{B0-isolated source cohort: subject replacement. Strict category replacement requires source category present in the source, target present after editing and source category absent after editing. Counts use all 48 attempts; Both requires both calibrated detectors. BG LPIPS uses 5\% source-box padding; parentheses give defined backgrounds. Lower LPIPS is useful only with successful editing.}\label{tab:b0-isolated-subject}
\begin{tabular}{rlrrrr}\toprule
Start & Method & DINO & OWLv2 & Both & BG LPIPS ($n$)\\\midrule
0 & B0-isolated learned & 19 & 48 & 19 & 0.173 (48)\\
0 & Native shared state & 9 & 48 & 9 & 0.451 (48)\\
0 & FlowEdit PNG, CFG4 & 9 & 48 & 9 & 0.297 (48)\\
0 & FlowEdit PNG, CFG6 & 11 & 48 & 11 & 0.355 (48)\\
0 & FlowEdit PNG, CFG8 & 13 & 48 & 13 & 0.397 (48)\\
0 & FlowEdit latent, CFG4 & 8 & 48 & 8 & 0.293 (48)\\
\midrule
6 & B0-isolated learned & 19 & 48 & 19 & 0.173 (48)\\
6 & Native shared state & 16 & 48 & 16 & 0.180 (48)\\
6 & FlowEdit PNG, CFG4 & 13 & 48 & 13 & 0.119 (48)\\
6 & FlowEdit PNG, CFG6 & 12 & 48 & 12 & 0.191 (48)\\
6 & FlowEdit PNG, CFG8 & 13 & 48 & 13 & 0.249 (48)\\
6 & FlowEdit latent, CFG4 & 7 & 48 & 7 & 0.112 (48)\\
\midrule
9 & B0-isolated learned & 21 & 48 & 21 & 0.111 (48)\\
9 & Native shared state & 17 & 48 & 17 & 0.112 (48)\\
9 & FlowEdit PNG, CFG4 & 12 & 47 & 12 & 0.070 (48)\\
9 & FlowEdit PNG, CFG6 & 18 & 48 & 18 & 0.113 (48)\\
9 & FlowEdit PNG, CFG8 & 14 & 48 & 14 & 0.161 (48)\\
9 & FlowEdit latent, CFG4 & 9 & 47 & 9 & 0.060 (48)\\
\midrule
15 & B0-isolated learned & 19 & 48 & 19 & 0.046 (48)\\
15 & Native shared state & 15 & 48 & 15 & 0.046 (48)\\
15 & FlowEdit PNG, CFG4 & 11 & 48 & 11 & 0.039 (48)\\
15 & FlowEdit PNG, CFG6 & 21 & 48 & 21 & 0.053 (48)\\
15 & FlowEdit PNG, CFG8 & 21 & 47 & 21 & 0.072 (48)\\
15 & FlowEdit latent, CFG4 & 9 & 48 & 9 & 0.026 (48)\\
\midrule
21 & B0-isolated learned & 5 & 38 & 5 & 0.016 (48)\\
21 & Native shared state & 3 & 41 & 3 & 0.017 (48)\\
21 & FlowEdit PNG, CFG4 & 2 & 36 & 2 & 0.026 (48)\\
21 & FlowEdit PNG, CFG6 & 9 & 45 & 9 & 0.031 (48)\\
21 & FlowEdit PNG, CFG8 & 11 & 48 & 11 & 0.038 (48)\\
21 & FlowEdit latent, CFG4 & 1 & 40 & 1 & 0.010 (48)\\
\bottomrule\end{tabular}\end{table}
Start 0 is full regeneration only for learned and native conditions, and for FlowEdit it is the start of its image-conditioned ODE. Isolated start 0 and start 6 give identical images in all 48 cases because their B0 prefix is identical.
\FloatBarrier
\paragraph{Matched background-change comparison.}
Subject-category retention is measured separately from the requested background change, the CLIP scene margin is an uncalibrated directional proxy, and region LPIPS is computed inside the source subject boxes, which also contain background pixels.
\begin{table}[!htbp]\centering\small
\setlength{\tabcolsep}{4pt}
\caption{B0-isolated source cohort: background changes. Retained counts require the original subject category in both source and output, using all 48 attempts. CLIP denotes scene-margin change; FG LPIPS uses 5\% padded source boxes, with defined-region counts in parentheses.}\label{tab:b0-isolated-background}
\begin{tabular}{rlrrrr}\toprule
Start & Method & DINO & OWLv2 & CLIP & FG LPIPS ($n$)\\\midrule
0 & B0-isolated learned & 48 & 48 & 0.103 & 0.657 (48)\\
0 & Native shared state & 48 & 48 & 0.101 & 0.664 (48)\\
0 & FlowEdit PNG, CFG4 & 48 & 48 & 0.089 & 0.572 (48)\\
0 & FlowEdit PNG, CFG6 & 48 & 48 & 0.092 & 0.605 (48)\\
0 & FlowEdit PNG, CFG8 & 48 & 48 & 0.094 & 0.642 (48)\\
0 & FlowEdit latent, CFG4 & 48 & 48 & 0.088 & 0.577 (48)\\
\midrule
6 & B0-isolated learned & 48 & 48 & 0.101 & 0.476 (48)\\
6 & Native shared state & 48 & 48 & 0.100 & 0.460 (48)\\
6 & FlowEdit PNG, CFG4 & 48 & 47 & 0.086 & 0.353 (48)\\
6 & FlowEdit PNG, CFG6 & 48 & 48 & 0.091 & 0.439 (48)\\
6 & FlowEdit PNG, CFG8 & 48 & 48 & 0.094 & 0.517 (48)\\
6 & FlowEdit latent, CFG4 & 48 & 48 & 0.086 & 0.373 (48)\\
\midrule
9 & B0-isolated learned & 48 & 48 & 0.077 & 0.325 (48)\\
9 & Native shared state & 48 & 48 & 0.076 & 0.336 (48)\\
9 & FlowEdit PNG, CFG4 & 48 & 48 & 0.052 & 0.241 (48)\\
9 & FlowEdit PNG, CFG6 & 48 & 48 & 0.076 & 0.319 (48)\\
9 & FlowEdit PNG, CFG8 & 48 & 48 & 0.086 & 0.381 (48)\\
9 & FlowEdit latent, CFG4 & 48 & 48 & 0.054 & 0.259 (48)\\
\midrule
15 & B0-isolated learned & 48 & 48 & 0.039 & 0.119 (48)\\
15 & Native shared state & 48 & 48 & 0.039 & 0.120 (48)\\
15 & FlowEdit PNG, CFG4 & 48 & 48 & 0.028 & 0.077 (48)\\
15 & FlowEdit PNG, CFG6 & 48 & 48 & 0.035 & 0.126 (48)\\
15 & FlowEdit PNG, CFG8 & 48 & 48 & 0.041 & 0.182 (48)\\
15 & FlowEdit latent, CFG4 & 48 & 48 & 0.030 & 0.065 (48)\\
\midrule
21 & B0-isolated learned & 48 & 48 & 0.015 & 0.032 (48)\\
21 & Native shared state & 48 & 48 & 0.016 & 0.035 (48)\\
21 & FlowEdit PNG, CFG4 & 48 & 48 & 0.011 & 0.041 (48)\\
21 & FlowEdit PNG, CFG6 & 48 & 48 & 0.016 & 0.061 (48)\\
21 & FlowEdit PNG, CFG8 & 48 & 48 & 0.020 & 0.090 (48)\\
21 & FlowEdit latent, CFG4 & 48 & 48 & 0.012 & 0.021 (48)\\
\bottomrule\end{tabular}\end{table}
\FloatBarrier
\paragraph{What exact invariance does and does not buy.}
Equation~\ref{eq:b0-input-invariance} holds bit-exactly, so the B0 boundary state is free of subject input. The rendered background can nevertheless change substantially. With the same seed, background sentence and B0 state, and only the subject text changed, the mean background-region LPIPS against the source is 0.1735 at 5\% padding (48/48 defined). This residual comes from B1/B2 acting on a shared spatial state. Isolating the input makes the early program subject-independent but does not by itself preserve the background, which is what resuming a saved prefix does.
\FloatBarrier
\paragraph{Recomputed source-size strata.}
Source size is recomputed on the new sources, since the intervention can change subject scale. At zero padding the frozen Grounding DINO boxes place 40/48 sources in the small stratum, 8/48 in the medium stratum and none in the large stratum. No large-source contrast is therefore available, and earlier large-source findings are neither reproduced nor contradicted.
\paragraph{Paired differences against the native shared-state control.}
Each row pairs the same case, direction and start; intervals are descriptive family bootstrap percentiles over the eight families, with no multiplicity or prior-search correction.
\begin{table}[!htbp]\centering\small
\setlength{\tabcolsep}{4pt}
\caption{B0-isolated learned minus native shared state, all 48 paired cases. Detector entries are percentage-point differences in strict replacement (subject) or source-subject retention (background); LPIPS entries are differences in the 5\% padded region metric, where a negative value is closer to the source. Intervals are unadjusted descriptive family bootstrap percentiles.}\label{tab:b0-isolated-paired}
\begin{tabular}{llrrr}\toprule
Direction & Start & Both/retained (pp) & 95\% interval & Region LPIPS\\\midrule
subject & 0 & +20.8 & [+4.2, +45.8] & -0.277\\
subject & 6 & +6.2 & [+0.0, +18.8] & -0.007\\
subject & 9 & +8.3 & [+2.1, +14.6] & -0.001\\
subject & 15 & +8.3 & [+2.1, +16.7] & -0.001\\
subject & 21 & +4.2 & [+0.0, +10.4] & -0.001\\
\midrule
background & 0 & +0.0 & [+0.0, +0.0] & -0.007\\
background & 6 & +0.0 & [+0.0, +0.0] & +0.016\\
background & 9 & +0.0 & [+0.0, +0.0] & -0.010\\
background & 15 & +0.0 & [+0.0, +0.0] & -0.001\\
background & 21 & +0.0 & [+0.0, +0.0] & -0.003\\
\bottomrule\end{tabular}\end{table}
At start 0 the native control regenerates rather than resuming a suffix, so its large LPIPS gap is not an editing result. From start 6 onward the two share the saved state, and most remaining differences are small relative to their intervals. The background rows are zero because both methods retain the source subject in all 48 cases, a ceiling of the retention proxy.
\FloatBarrier
\paragraph{Distribution-quality scope.}
MJHQ-30K supplies full captions but no background descriptions, so the FID table does not cover this variant and no earlier value is reused for it. A preflight of the frozen automatic scene compiler found one empty B0 text among the first 32 sorted MJHQ captions, and only 3/8 subject-replacement pairs yielded identical scene text. The automatic route therefore fails the input-invariance contract, and no FID is reported for the variant.

\FloatBarrier

\clearpage
\part*{Part V\quad Stress tests and example galleries}
\section{Bidirectional editing stress test}\label{app:bidirectional-stress}
\paragraph{Frozen design.}
We also build 48 sources from four subject pairs
(dog--cat, cow--horse, elephant--giraffe, horse--zebra), two scene pairs
(meadow--snow field, beach--rocky desert), two framings and three seeds. Each
source receives a subject-only and a background-only edit under six methods
at branches 9, 15 and 21, for 1728 outputs. The learned route uses the final
6/12/12 checkpoint, sources come from the frozen reference compiler, and
native continuation shares the saved state. FlowEdit--SANA uses source CFG 4
and target CFG 4, 6 or 8, and a latent-input CFG 4 control removes the
PNG/VAE round trip.

\paragraph{Measurements.}
Strict replacement uses the frozen DINO (0.40) and OWLv2 (0.31) thresholds
separately and jointly. Source-box regions with 5\% padding give background
LPIPS for subject edits and subject-region LPIPS for background edits, and
all 48 sources have defined regions. The scene-direction CLIP margin is a
continuous proxy, not a calibrated success rate. Subject retention under
background edits is 48/48 for every method and branch, so it cannot
distinguish methods.

\begin{table}[!htbp]
\centering
\caption{Complete bidirectional grid, 48 attempts per direction and row.
D/O/B are strict replacement counts for DINO, OWLv2 and both. BG and FG are
non-target LPIPS for subject and background edits (lower is better), and
Scene $\Delta$ is the CLIP margin change.}
\label{tab:bidirectional-stress}
\small\setlength{\tabcolsep}{3.5pt}
\begin{tabular}{rlrrrr}
\toprule
Branch & Method & D/O/B & BG $\downarrow$ & FG $\downarrow$ & Scene $\Delta$ \\
\midrule
9 & Learned suffix & 13/48/13 & 0.1195 & 0.2953 & 0.0844 \\
9 & Native shared-state & 8/48/8 & 0.1210 & 0.3095 & 0.0826 \\
9 & FlowEdit PNG, CFG 4 & 10/48/10 & 0.0762 & 0.1890 & 0.0597 \\
9 & FlowEdit PNG, CFG 6 & 14/48/14 & 0.1059 & 0.2424 & 0.0816 \\
9 & FlowEdit PNG, CFG 8 & 12/48/12 & 0.1407 & 0.3065 & 0.0869 \\
9 & FlowEdit latent, CFG 4 & 10/48/10 & 0.0697 & 0.2002 & 0.0637 \\
\addlinespace
15 & Learned suffix & 9/45/9 & 0.0548 & 0.1030 & 0.0394 \\
15 & Native shared-state & 6/46/6 & 0.0550 & 0.0995 & 0.0386 \\
15 & FlowEdit PNG, CFG 4 & 6/43/6 & 0.0398 & 0.0658 & 0.0252 \\
15 & FlowEdit PNG, CFG 6 & 10/48/10 & 0.0534 & 0.1057 & 0.0322 \\
15 & FlowEdit PNG, CFG 8 & 12/47/12 & 0.0707 & 0.1570 & 0.0389 \\
15 & FlowEdit latent, CFG 4 & 5/44/5 & 0.0295 & 0.0572 & 0.0254 \\
\addlinespace
21 & Learned suffix & 2/25/2 & 0.0177 & 0.0331 & 0.0149 \\
21 & Native shared-state & 3/26/3 & 0.0172 & 0.0309 & 0.0157 \\
21 & FlowEdit PNG, CFG 4 & 3/22/3 & 0.0249 & 0.0366 & 0.0104 \\
21 & FlowEdit PNG, CFG 6 & 4/31/4 & 0.0299 & 0.0542 & 0.0161 \\
21 & FlowEdit PNG, CFG 8 & 4/35/4 & 0.0376 & 0.0801 & 0.0210 \\
21 & FlowEdit latent, CFG 4 & 2/23/2 & 0.0113 & 0.0176 & 0.0106 \\
\bottomrule
\end{tabular}
\end{table}

\paragraph{Interpretation.}
At branch 9, learned versus native continuation yields 13 versus 8 strict
joint replacements, with similar background LPIPS (0.1195 versus 0.1210).
The paired replacement difference is 10.4 percentage points, with a
95\% family-bootstrap interval of [2.1,22.9]; the background-LPIPS difference
is $-0.0015$ with interval $[-0.0044,0.0014]$. For background edits at the
same branch, subject-region LPIPS is 0.2953 versus 0.3095 (difference
$-0.0142$, interval $[-0.0289,-0.0034]$), while the scene-direction proxy
is slightly higher. Intervals use 10,000 resamples of the eight families.
These are local improvements over native continuation. FlowEdit CFG 6
obtains 14 joint replacements with lower background LPIPS (0.1059) at
branch 9, so they do not establish superiority over the full FlowEdit
grid.

\paragraph{Source size and visual checks.}
Source-box areas define small ($<10\%$, $n=11$), medium ($10$--$25\%$,
$n=14$) and large ($\ge25\%$, $n=23$) strata. At branch 9 the large-group
joint counts are 2, 2, 1, 3, 3 and 2 in table order, so a larger-subject
advantage is not supported. OWLv2 replacement saturates at branch 9, and
DINO has a calibration shortfall for the target class in 12/48 cases (8/23
large). Visual inspection finds stronger late subject changes, but also
hybrid anatomy, background edits that alter pose and late scene edits with
little visible effect.

\begin{figure}[!htbp]
\centering
\includegraphics[width=\linewidth]{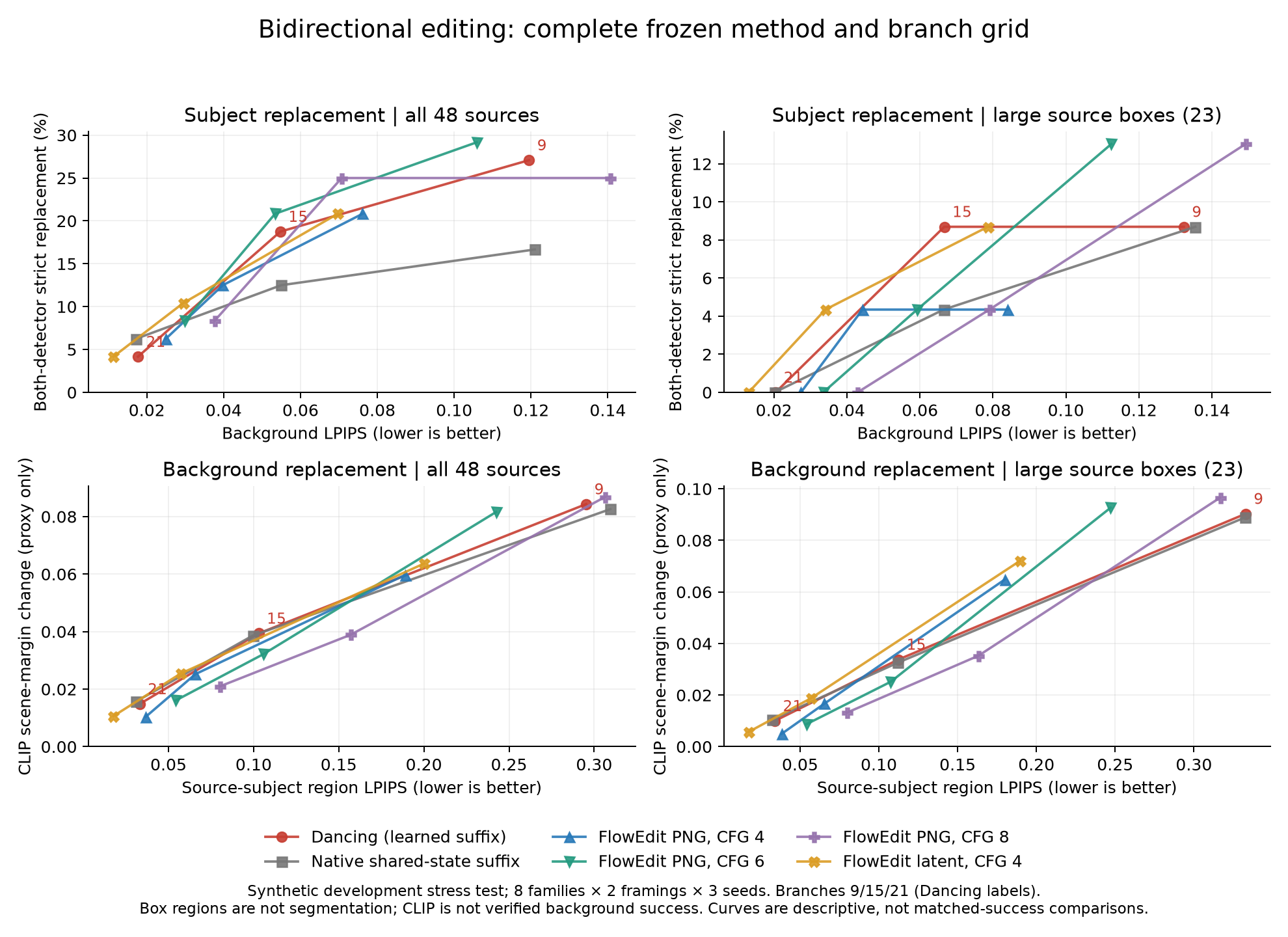}
\caption{Bidirectional stress-test tradeoffs for all sources and for large
source boxes. Lines connect the three measured branch points. Upper panels
use strict replacement, lower panels the scene-direction CLIP change.}
\label{fig:bidirectional-stress}
\end{figure}

\FloatBarrier
\section{Additional final-checkpoint examples}
\label{app:final-representative}
The examples below use the final compiler with schedule 6/12/12
at branch 9. In each row both target outputs share the compiler and initial
noise, and prefix reuse resumes the saved source state. They were selected
after reviewing 120 source/regeneration/edit triplets (all 72 family--seed
pairs for this schedule and seed 101 for the other two) to show where prefix
reuse visibly retains context that full regeneration loses. Images are
uncropped and unretouched, and the selection does not estimate success
frequency. Failures also occur. In a surfboard-to-kayak edit the person
stands on the kayak instead of holding it, as do the image-input controls,
and a racket-to-bat edit alters the court and duplicates the implement.

\begin{figure}[!htbp]
\centering\setlength{\tabcolsep}{2pt}
\begin{tabular}{ccc}
\scriptsize Source & \scriptsize Full regeneration & \scriptsize Ours: prefix reuse\\
\multicolumn{3}{l}{\scriptsize Horse to donkey; seed 303}\\
\includegraphics[width=.31\linewidth]{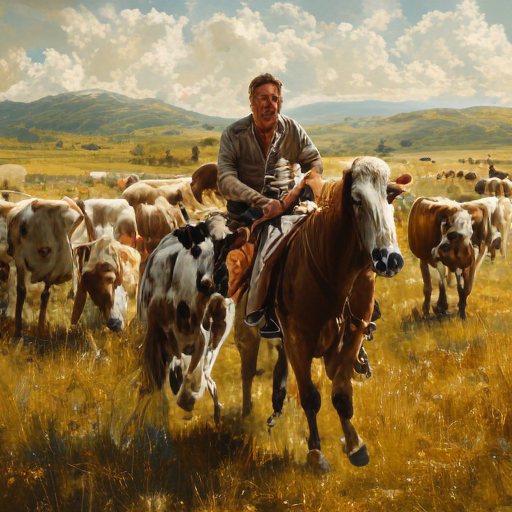} & \includegraphics[width=.31\linewidth]{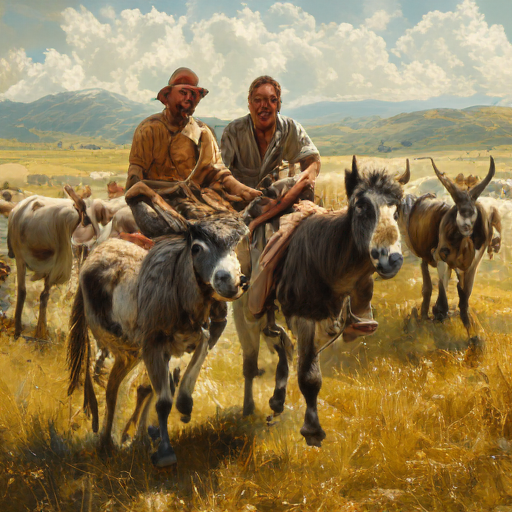} & \includegraphics[width=.31\linewidth]{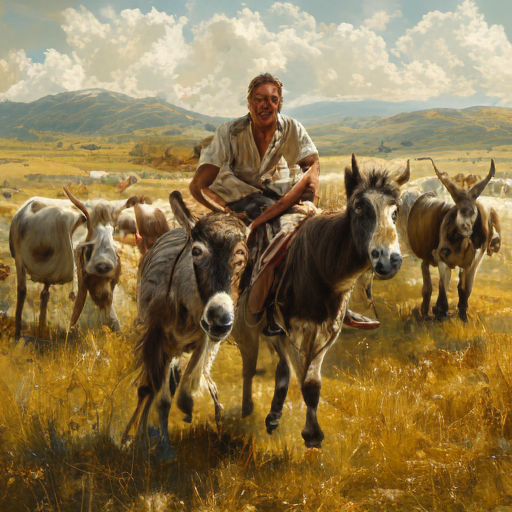}\\[3pt]
\multicolumn{3}{l}{\scriptsize Dog to cat; seed 303}\\
\includegraphics[width=.31\linewidth]{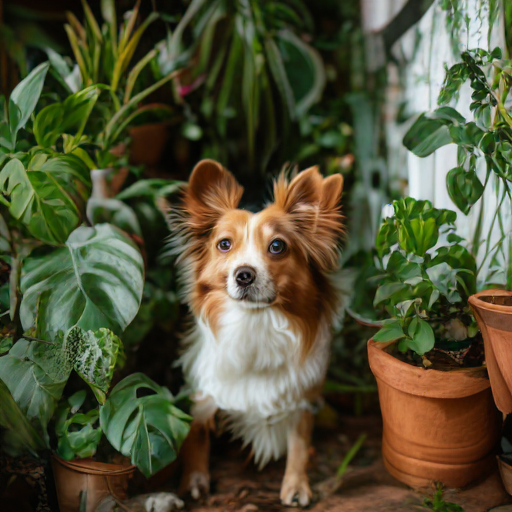} & \includegraphics[width=.31\linewidth]{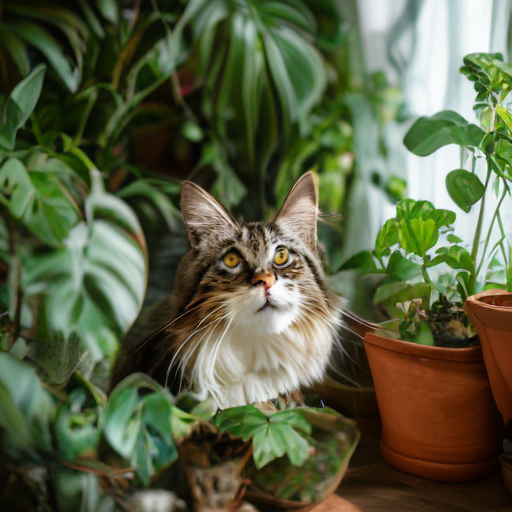} & \includegraphics[width=.31\linewidth]{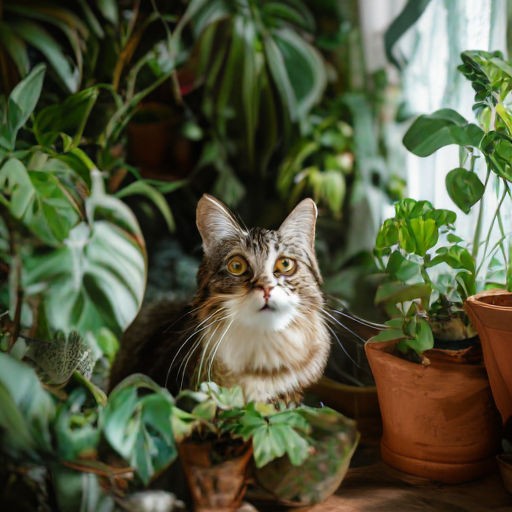}\\[3pt]
\multicolumn{3}{l}{\scriptsize Teddy bear to plush rabbit; seed 202}\\
\includegraphics[width=.31\linewidth]{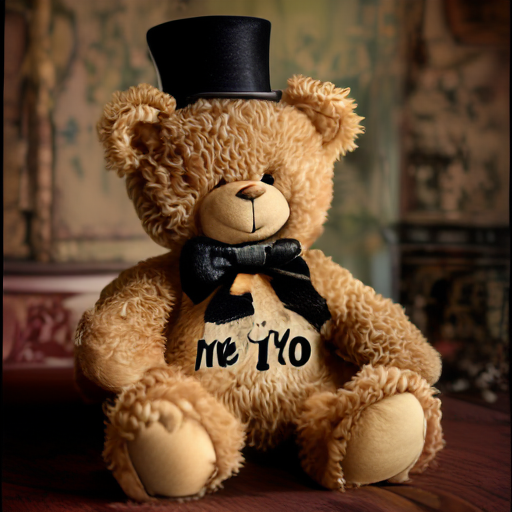} & \includegraphics[width=.31\linewidth]{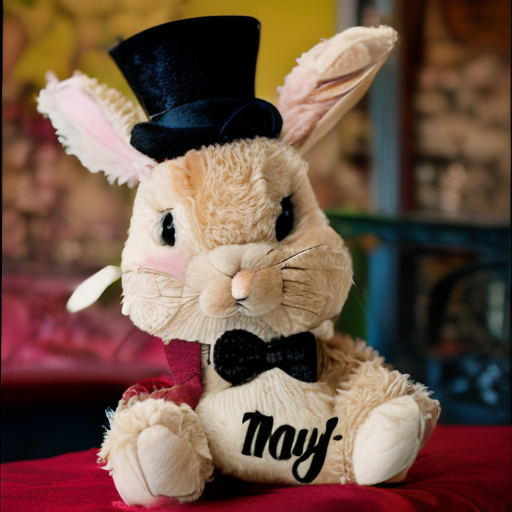} & \includegraphics[width=.31\linewidth]{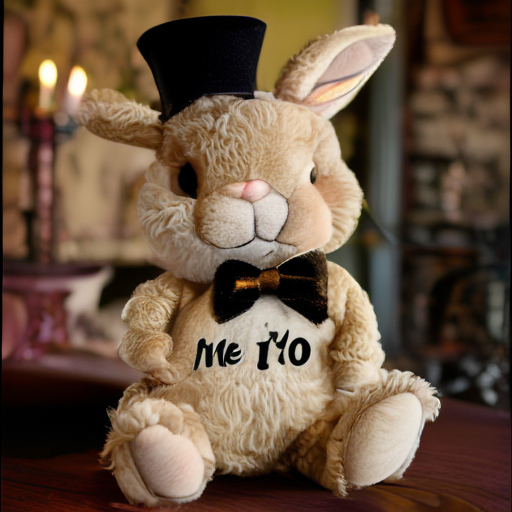}\\
\end{tabular}
\caption{\textbf{Prefix reuse retains what regeneration discards.}
Regeneration adds a second rider, reframes the plants and replaces the
room and the lettering. Prefix reuse keeps the single rider and herd, the
potted plants and their layout, and the candlelit room, the bow tie and the
``me to you'' lettering, while changing the requested subject.}
\label{fig:final-representative-1}
\end{figure}

\end{document}